\documentclass[acmsmall,screen,balance=false]{acmart} 
\usepackage{graphicx}
\usepackage{amsfonts}
\usepackage{hyperref}
\usepackage{url}
\usepackage{booktabs}
\usepackage{nicefrac}
\usepackage{microtype}
\usepackage{comment}
\usepackage{lipsum}
\graphicspath{ {./} }
\usepackage[ruled,vlined,linesnumbered]{algorithm2e}
\usepackage{mathtools}
\usepackage{xcolor}
\usepackage{array}
\usepackage{cleveref}
\usepackage{enumitem}
\usepackage{wrapfig}
\usepackage{booktabs, tabularx}
\usepackage{mathtools,trimclip,stackengine,scalerel}
\usepackage{eucal}
\usepackage{caption}
\graphicspath{ {./images/} }
\usepackage{ifthen}
\usepackage{amsmath}
\usepackage{comment}
\usepackage{thmtools, thm-restate}
\usepackage{hyperref}
\usepackage{longtable}
\usepackage{pifont}
\newcommand{\cmark}{\ding{51}}%
\newcommand{\xmark}{\ding{55}}%
\usepackage{environ}

\newtheorem{problem}{Problem}
\SetKwInOut{Parameter}{Parameters}

\newcommand{\propl}{individual differential privacy deterministic bound\xspace}
\newcommand{\propi}{iDP\xspace}
\newcommand{\propa}{iDP-DB\xspace}

\newcommand{\tool}{\texttt{LUCID}\xspace}
\newcommand{\boundtool}{\texttt{LUCID-Compute}\xspace}
\newcommand{\reptool}{\texttt{LUCID-Inference}\xspace}

\newcommand{\anan}[1]{\textcolor{blue}{\bf Anan: #1}}
\newcommand{\TBD}[1]{\textcolor{red}{\bf TBD}}
\newcommand{\WITHNAIVE}{1}
\newcommand{\EXTENDEDVER}{-1} 

\begin{document}

\title{Guarding the Privacy of Label-Only Access to Neural Network Classifiers via \propi Verification}

\author{Anan Kabaha}
\orcid{0000-0002-0969-6169}
\affiliation{%
  \institution{Technion}
  \city{Haifa}
  \country{Israel}
}
\email{anan.kabaha@campus.technion.ac.il}

\author{Dana Drachsler-Cohen}
\orcid{0000-0001-6644-5377}
\affiliation{%
  \institution{Technion}
  \city{Haifa}
  \country{Israel}
}
\email{ddana@ee.technion.ac.il}



\begin{abstract}
Neural networks are susceptible to privacy attacks that can extract private information of the training set. To cope, several training algorithms guarantee differential privacy (DP) by adding noise to their computation. 
However, DP requires to add noise 
considering \emph{every possible} training set. 
This leads to a significant decrease in the network's accuracy.
Individual DP (\propi) restricts DP to a \emph{given} training set. 
We observe that some inputs deterministically satisfy \propi \emph{without any noise}. By identifying them, we can provide \propi label-only access to the network with a minor decrease to its accuracy. 
However, identifying the inputs that satisfy \propi without any noise is highly challenging. 
Our key idea is to compute the \emph{\propi deterministic bound} (\propa), which overapproximates the set of inputs that do not satisfy \propi, and add noise only to their predicted labels. 
To compute the tightest \propa, which enables to guard the label-only access with minimal accuracy decrease, we propose \tool, which leverages several formal verification techniques.
First, it encodes the problem as a mixed-integer linear program, defined over a network and over every network trained identically but without a unique data point. 
Second, 
it abstracts a set of networks using a \emph{hyper-network}. 
Third, it eliminates the overapproximation error via a novel branch-and-bound technique. 
Fourth, it bounds the differences of matching neurons in the network and the hyper-network, encodes them as linear constraints to prune the search space, and employs linear relaxation if they are small.  
We evaluate \tool on fully-connected and convolutional networks for four datasets and compare the results to existing DP training algorithms, which in particular provide \propi guarantees.  
We show that \tool can provide classifiers with a perfect individuals' privacy guarantee ($0$-\propi) -- which is infeasible for DP training algorithms -- with an accuracy decrease of 1.4\%. For more relaxed $\varepsilon$-\propi guarantees, \tool has an accuracy decrease of 1.2\%.  
In contrast, existing DP training algorithms that obtain $\varepsilon$-DP guarantees, and in particular $\varepsilon$-\propi guarantees, reduce the accuracy by 12.7\%. 
\end{abstract}

\maketitle
\section{Introduction}
Deep neural networks achieve remarkable success in a wide range of tasks. 
However, this success is challenged by their vulnerability to privacy leakage~\citep{ref_23,ref_24,ref_25}, a critical concern in the era of data-driven algorithms. 
Privacy leakage refers to the unintended disclosure of sensitive information of the training set. 
Several kinds of attacks demonstrate this vulnerability: 
membership inference attacks
~\citep{ref_26, ref_20, ref_28,ref_13}, 
reconstruction attacks
~\citep{ref_29,ref_30,ref_31}, and
model extraction attacks
~\citep{ref83,ref84,ref85}. 
Some of these attacks can succeed even in the most challenging scenario, where the attacker has only label-only access to the network~\citep{ref_12,ref_13,ref_14}.    
This label-only access setting is common 
in machine-learning-as-a-service platforms (e.g., Google’s Vertex AI~\citep{ref_103} or Amazon’s ML on AWS~\citep{ref_104}). 
Thus, it is essential to develop privacy-preserving mechanisms.

There are different approaches for mitigating privacy leakage, including privacy-aware searches for a network architecture~\cite{ref_16,ref_17}, integration of regularization terms during training to prevent overfitting~\cite{ref_18, ref_19}, 
or federated learning over local datasets~\cite{ref_34,ref_57}. However, these approaches typically do not have a formal privacy guarantee. 
One of the most popular formal guarantees for privacy is 
differential privacy (DP)~\cite{Dwork06}. 
An algorithm taking as input a dataset is DP if, for any dataset,
 the presence or absence of a single data point does not significantly change its output. 
Several training algorithms add noise to guarantee DP~\cite{ref_22,ref_36,ref_37,ref_58,ref_59,ref_60,ref_61,ref_62}.  
Many of them rely on \emph{DP-SGD}~\citep{ref_22,ref_36,ref_37,ref_58}, where noise is added during training to the parameters' gradients.  
However, 
they have several disadvantages.
Their main disadvantage is that to satisfy DP, they add noise calibrated to protect the privacy of \emph{any} training set. To this end, they add noise protecting the \emph{maximal} privacy leakage, which tends to be high and thus decreases the resulting network's accuracy. 
In many scenarios, network designers care about protecting the privacy of a \emph{given} dataset.
For example, a network designer wishing to protect the privacy of a patients' dataset may not care that the training algorithm protects the privacy of other datasets (e.g., image datasets). 
However, even if they do not care about these datasets, ensuring DP forces the training algorithm
to add significantly higher noise and, consequently, the accuracy of the network that the designer cares about can be substantially reduced. 
Another disadvantage is that DP-SGD algorithms add noise to a large number of computations, thus 
the amount of added noise is often higher than necessary.  
This is because the noise 
is a function of the user-defined allowed privacy leakage and the algorithm's actual privacy leakage which is commonly overapproximated (e.g., using composition theorems~\cite{ref_62,ref_63,ref_64}).
Also, DP-SGD algorithms are specific algorithms and cannot easily be combined with other training techniques to leverage parallel advances in training. 
This raises the question: 
\emph{can we provide privacy guarantees for individuals in a given dataset used by a training algorithm while introducing a minor accuracy decrease to the label predictions of the resulting network?}

We propose label-only access to a network that guarantees \emph{Individual Differential Privacy} (\propi)~\citep{ref_88,ref_90} with a minor accuracy decrease.
\propi enforces the DP requirement only for individuals in a \emph{given} dataset.
Thanks to this relaxation, it enables the algorithm to introduce a lower noise and thus reduce the accuracy decrease. 
Although several other relaxations of DP have been proposed in order to reduce its accuracy decrease, they target specialized settings, which pose additional requirements. For example, individualized privacy assignment~\cite{ref_94,ref_95} requires to assign different privacy budgets to every individual in the training set, depending on how much the individual wishes to protect their privacy, which is not always given or easy to quantify. 
Local differential privacy~\cite{ref_60,ref_97} requires to add randomization to the individuals' data before being released to form the training set, which assumes that individuals' data is collected by a certain algorithm. Renyi differential privacy~\cite{ref_98,ref_99} changes the DP requirement based on the Renyi divergence. Homomorphic-encryption approaches~\cite{ref_100,ref_101} require to encrypt the data. In contrast, \propi assumes the exact setting
as DP, except that it limits the guarantee to the individuals in a given dataset. Additionally, \propi 
has similar properties to DP (e.g., parallel and sequential composition) and common DP noise mechanisms naturally extend to \propi. 
To the best of our knowledge, we are the first to enforce \propi for neural networks. 

 \begin{figure*}[t]
    \centering
  \includegraphics[width=1\linewidth, trim=0 280 0 0, clip,page=10]{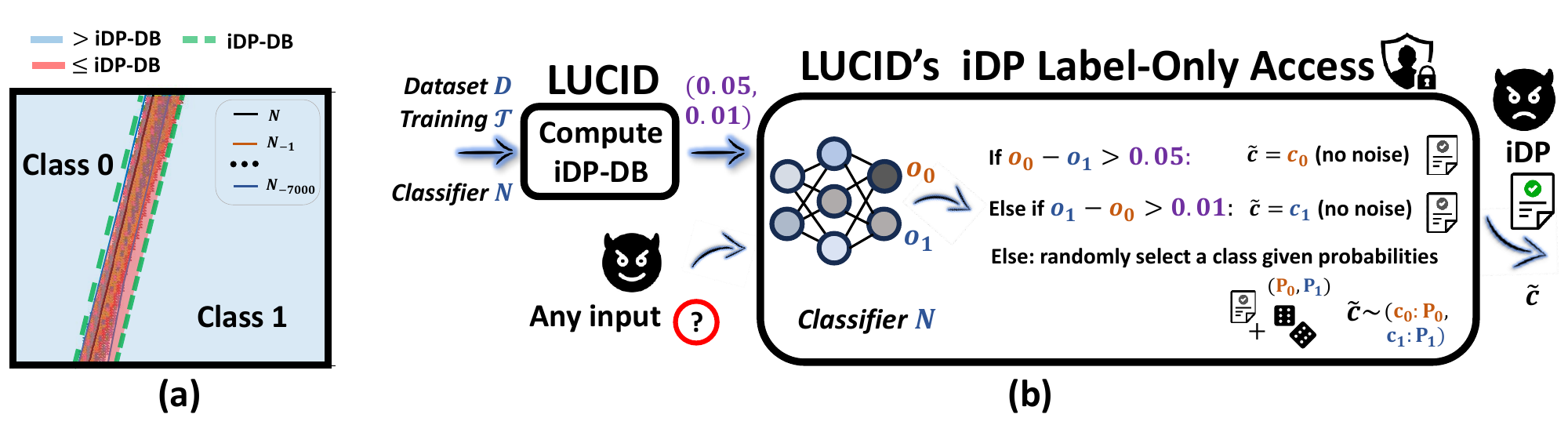}
    \caption{(a)~Illustration of the subspaces of inputs with confidence over or below the \propa on a 2D synthetic training set comprising 7,000 data points. (b)~An overview of \tool:
     given a dataset~$D$, a training algorithm~$\mathcal{T}$, and a classifier~$N$, it computes the \propa of every label (in this example, {\propa}$_{c_0}=0.05$, {\propa}$_{c_1}=0.01$) 
     and returns
     \propi label-only access to $N$. Given an input $x$, this access computes $N(x)$ to identify the predicted label $c$ and the confidence $\mathcal{C}$ in $c$.
      If $\mathcal{C}>$ {\propa}$_c$, it returns $c$; otherwise, it employs a noise mechanism. In this example, if a user submits to the \propi label-only access $x$ where
      $N(x) = (0.6,0.4)$, 
       then $c=c_0$ and $\mathcal{C}=0.6-0.4>0.05$ and thus $c_0$ is returned. If $N(x) = (0.52,0.48)$, then $c=c_0$ and $\mathcal{C}=0.52-0.48\leq 0.05$ and thus a noise mechanism is used to select a class $\widetilde{c}\in\{c_0,c_1\}$ given a probability distribution $(P_0,P_1)$.} 
    \label{fig::intro}
\end{figure*}


\sloppy
To naively obtain \propi label-only access to a network, we could add
noise that depends on the maximum privacy leakage of any possible input to the network. 
However, this would significantly decrease the accuracy.
Instead, we observe that 
not all inputs to the network lead to an \propi privacy leakage. By identifying these inputs, we can avoid adding noise to their predicted labels and thus avoid unnecessary decrease in accuracy. However, determining whether an input leads to an \propi privacy leakage is challenging since it requires to obtain its classification as determined by a large number of networks: the network trained on the full training set and every network trained on the full training set except for a unique data point. Instead, we propose to overapproximate the set of privacy leaking inputs with the \emph{\propl} (\propa). 
The \propa is the maximal classification confidence of inputs that violate \propi. \emph{Any} input that the network classifies with confidence over this bound does not lead to an \propi violation. 
Namely, such inputs are classified the same regardless if the classifier is trained on the full training set or if any data point is omitted. Given this bound, we can design \propi label-only access to the network by adding noise only to the predicted labels of inputs that the network classifies with confidence smaller or equal to the \propa. 
\Cref{fig::intro}(a) visualizes the \propa.
In this example, we consider a binary network classifier trained on a 2D synthetic training set comprising 7,000 data points. 
The figure shows the decision boundary of the classifier along with the decision boundaries of every 
classifier trained identically except that its training set excludes a unique data point (the colored lines in~\Cref{fig::intro}(a)). 
Label-only access to this classifier is \propi
if all classifiers classify every input the same. 
However, in practice, these classifiers are different. Namely, their decision boundaries are close but not identical. Because the decision boundaries are close, there is a large subspace of inputs that are classified the same by all classifiers (shown by the light blue background). These inputs' classification confidence is higher than the \propa (shown by the dashed green lines), and for them, \propi holds \emph{without added noise}. 

Building on this idea, we design
\tool (for \textbf{L}abel-G\textbf{U}arded \textbf{C}lassifier by \textbf{I}ndividual \textbf{D}ifferential Privacy), a system that computes the \propa of every label and then returns \propi label-only access that wraps a trained network. 
Given an input, \tool's label-only access passes the input through the network. If the classification confidence is over the \propa, it returns the predicted label. If not, it selects a label to return using a noise mechanism.  
\Cref{fig::intro}(b) visualizes \tool. 
The main benefits of our \propi label-only access are: (1)~it is applicable to any training algorithm, 
(2)~it identifies inputs that cannot lead to privacy leakage and does not add noise to their predicted labels, which leads to a minor decrease in the network's accuracy,
(3)~its noise mechanism can be easily configured with the desired level of privacy guarantee, providing a privacy-accuracy trade-off,
(4)~its privacy guarantee is not tied to a specific privacy attack,  
 and (5)~it is fully automatic.

Computing the exact \propa is highly challenging for several reasons. 
First, it is defined over a very large number of classifiers (commonly, several thousand): the classifier trained with the full training set and every classifier trained on the full training set except for a unique data point.
Second, the \propa is a global property requiring to determine for \emph{any} input if it may be classified differently by any of these classifiers and thus violate \propi. Third, it requires to determine the \emph{maximal} classification confidence of any input that violates \propi. 
To compute the tightest \propa, \tool relies on several formal verification techniques. 
First, it encodes the problem as a
mixed-integer linear program (MILP), defined over the network and every network trained on the same training set except for a unique data point. 
Second, to scale the analysis over the immense number of networks, \tool abstracts a set of networks using a \emph{hyper-network}. 
It then restates the MILP to compute the \propa of the network and a hyper-network. 
This enables \tool a simultaneous analysis of multiple, closely related networks. 
Third, to eliminate the abstraction's overapproximation error, \tool 
relies on a novel branch-and-bound (BaB) technique. 
Our BaB branches by refining a hyper-network into several hyper-networks, each
abstracting close networks that are identified by clustering. Our BaB bounds by identifying hyper-networks whose \propa can bound the \propa of other hyper-networks. To keep the number of analyzed hyper-networks minimal, \tool further orders the branching of hyper-networks to increase the chances of bounding as many hyper-networks as possible.
Fourth, to prune the MILP's search space, \tool bounds the differences of matching neurons in the network and the hyper-network and adds them as linear constraints. 
Fifth, neurons in the hyper-network whose differences are small are overapproximated by linear relaxation,
to reduce the MILP's complexity. While this step introduces an overapproximation error, it is small in practice and it is configurable. 

We evaluate \tool over four data-sensitive datasets (Cryptojacking, Twitter Spam Accounts, Adult Census, and Default of Credit Card Clients) and over fully-connected and convolutional classifiers. 
We compare it to two DP-SGD techniques: the original one~\cite{ref_22}, providing a DP guarantee for all the classifier's parameters, and 
ALIBI~\cite{ref_58}, providing a DP guarantee for the classifier's output. 
Our results indicate that thanks to our formal verification analysis, \tool can provide a perfect \propi guarantee ($\varepsilon=0$) for label-only access with 1.4\% accuracy decrease on average. 
For more relaxed \propi guarantees (higher $\varepsilon$), \tool provides \propi label-only access with 1.2\% accuracy decrease on average. 
In contrast, the DP baselines decrease the accuracy on average by 12.7\% to obtain an $\varepsilon$-DP guarantee, which in particular provides $\varepsilon$-\propi guarantee. 

In summary, our main contributions are: (1)~formalizing \propi for label-only access to networks, enabling a minor accuracy decrease, (2)~defining the \propa that overapproximates inputs violating \propi, (3)~designing a system called \tool for computing the \propa, relying on several formal verification techniques to scale (constraint solving, hyper-networks, branch-and-bound, 
pruning by bounding the differences of matching neurons, and linear relaxation) (4)~an extensive evaluation showing that \tool provides ~\propi label-only access with less than 1.4\% accuracy decrease. 


\section{Preliminaries}
\label{sec:preliminary}
This section provides background on neural network classifiers and individual differential privacy.
  
\paragraph{Neural network classifiers}
A classifier maps an input into a score vector over a set of labels (also called classes) $C$, i.e., 
it is a function: $N:[0,1]^d\to \mathbb{R}^{|C|}$. 
Given an input $x$, its classification is the label with the maximal score: $c= \text{argmax}(N(x))$. 
We focus on classifiers implemented as neural networks.
A neural network consists of an input layer followed by $L$ layers, where the last layer is the output layer. 
Layers consist of neurons, each connected to some or all neurons in the next layer. 
A neuron computes a function, which is typically a linear combination of its input neurons followed by a non-linear activation function. 
The output layer consists of $|C|$ neurons, each outputs the score of one of the labels.
We denote by $z_{m,k}$ the $k^\text{th}$ neuron in layer $m$, for $m\in \{0,\ldots,L\}$ (where $0$ denotes the input layer) and $k\in [k_m]=\{1,\ldots,k_m\}$. 
The input layer $z_0$ passes the input $x$ to the next layer (i.e., $z_{0,k}=x_k$ for all $k\in [d]$).
 The neurons compute a function determined by the layer type. 
  Our definitions focus on fully-connected layers, but our implementation also supports convolutional layers and it is easily extensible to other layers, such as max-pooling layers~\citep{ref_41} and residual layers~\citep{ref_40}. 
  In a fully-connected layer, a neuron $z_{m,k}$ 
   gets as input the outputs of all neurons in the preceding layer. 
   It has a weight for each input ${w}_{m,k,k'}$ and a bias $b_{m,k}$. Its function is the weighted sum of its inputs and bias followed by an activation function. We focus on the ReLU activation function. 
   Formally, the function of a non-input neuron is $z_{m,k}=\text{ReLU}(\hat{z}_{m,k})=\max(0,\hat{z}_{m,k})$, where $\hat{z}_{m,k}$
   is the weighted sum:
 $\hat{z}_{m,k}=b_{m,k}+\sum_{k'=1}^{k_{m-1}}{w}_{m,k,k'}\cdot{z}_{m-1,k'}$. 
%
The network's parameters -- the weights and biases -- are determined by a training algorithm. We focus on supervised learning, where the training algorithm~$\mathcal{T}$ is provided with an initial network $\widetilde{N}$ (e.g., with random parameter values), called the network architecture, and a dataset $D \subseteq [0,1]^d \times C$ of labeled data points $(x_D,y_D)$. In the following, we abbreviate a pair $(x_D,y_D)$ by $x_D$. 
The exact computation of the training algorithm  $\mathcal{T}$  is irrelevant to our approach. In particular, it may run a long series of iterations, each consisting of a large number of computations (e.g., over the parameters' gradients).  
The important point in our context is that given a network architecture $\widetilde{N}$ and a random seed (in case $\mathcal{T}$ has randomized choices), the training algorithm is a deterministic function mapping a training set $D$ into a neural network.  
\paragraph{Differential privacy} 
Differential privacy (DP)~\cite{Dwork06} is a popular privacy property designed to protect the privacy of individual data points in a dataset, which is being accessed by one or more functions $f$. 
The original DP definition assumes that $f$ is invoked interactively by the user. 
To avoid privacy leakage, $f$ is protected by a mechanism $\mathcal{A}$, which is a randomized version of $f$ adding random noise to its computations.
The added noise ensures that the inclusion or exclusion of any single data point in the dataset does not significantly affect the function's output. Formally:

\begin{definition}[Differential Privacy]~\label{def:differential_privacy}
A randomized mechanism $\mathcal{A}$, over a domain of datasets $\mathcal{D}$, is $\varepsilon$-differentially private if for any two adjacent datasets $D,D'\in \mathcal{D}$ (differing by one data point), for all sets of observed outputs $\mathcal{O} \subseteq \text{Range}(\mathcal{A})$ we have:  $\Pr[\mathcal{A}(D) \in \mathcal{O}] \leq e^{\varepsilon} \cdot \Pr[\mathcal{A}(D') \in \mathcal{O}]$. 
\end{definition}
In this definition, $\varepsilon$ is a small non-negative number, bounding the privacy loss. 
A common approach to create a randomized version of $f$ is by adding noise that depends on $f$'s \emph{global sensitivity}
$S(f)$~\citep{Dwork06,ref_22,ref_58,ref_36,ref_37}. The global sensitivity is the maximum difference of $f$ over any two adjacent datasets in $\mathcal{D}$. Formally: 

\begin{definition}[Global Sensitivity]~\label{def:globall_sensitivity}
The global sensitivity of a function $f:\mathcal{D}\rightarrow \mathbb{R}^n$ is \\
$S(f)=\max_{D,D'\in \mathcal{D} s.t. |D-D'|= 1}{||f(D)-f(D')||}$.
\end{definition}

Differential privacy and global sensitivity are defined with respect to any two adjacent datasets over the dataset domain $\mathcal{D}$. 
While such definitions are useful in many real-world scenarios, in the context of machine learning, it may be too restrictive. 
While there may be scenarios where network designers wish to design a DP training algorithm, protecting the privacy of any dataset, in other scenarios, network designers may only care about protecting the privacy of data points in a \emph{given} dataset. 
For example, a network designer that has collected private information of patients towards designing a network predicting whether a new person has a certain disease may not care whether their training algorithm does not leak private information of the MNIST image dataset~\citep{ref_41} (or any other dataset for this matter). However, even if they do not care about these datasets, ensuring DP may force them 
to add a significantly higher noise because these irrelevant datasets may increase the global sensitivity of the training algorithm. Consequently, the accuracy of the network that the designer cares about can be substantially reduced. In fact, as we demonstrate in~\Cref{sec:eval}, in some cases, the network's classification accuracy can be reduced to that of a random classifier.
In such scenarios, it may be better to relax DP and ensure the privacy of data points in a \emph{given} dataset. This property is called \emph{individual differential privacy}.

\paragraph{Individual differential privacy} 
Individual Differential Privacy (\propi)~\citep{ref_88} is a relaxation of DP, designed to protect the privacy of individual data points in a \emph{given} dataset. 
Similarly to DP, a randomized mechanism satisfies \propi if its output is not significantly affected 
by the inclusion or exclusion of any single data point in the given dataset. Formally:
\begin{definition}[Individual Differential Privacy]~\label{def:individual_differential_privacy}
Given a dataset $D\in \mathcal{D}$, a randomized mechanism $\mathcal{A}$, over a domain of datasets $\mathcal{D}$, is $\varepsilon$-individually differentially private if for any dataset $D'$ adjacent to $D$ (differing by one data point), for all sets of observed outputs $\mathcal{O} \subseteq \text{Range}(\mathcal{A})$ we have:  \\
 $e^{-\varepsilon} \cdot \Pr[\mathcal{A}(D') \in \mathcal{O}] \leq\Pr[\mathcal{A}(D) \in \mathcal{O}] \leq e^{\varepsilon} \cdot \Pr[\mathcal{A}(D') \in \mathcal{O}]$. 
\end{definition}
Similarly to DP, a randomized mechanism for $f$ can be obtained by adding noise that depends on $f$'s \emph{local sensitivity} 
$S_L(f)$. The local sensitivity is the maximum difference of $f$ with respect to a given dataset $D$ and any of its adjacent datasets. Formally: 

\begin{definition}[Local Sensitivity]~\label{def:local_sensitivity}
Given a dataset $D\in \mathcal{D}$, 
the local sensitivity of a function $f:\mathcal{D}\rightarrow \mathbb{R}^n$ is 
$S_L(f)=\max_{D'\in \mathcal{D} s.t. |D-D'|= 1}{||f(D)-f(D')||}$.
\end{definition}

\propi has been shown to have several important properties that hold for DP, including parallel and sequential composition~\citep{ref_88}.
Additionally, several DP noise mechanisms naturally extend to \propi by replacing the global sensitivity with the local sensitivity (e.g., the Laplace noise mechanism~\citep{ref_88}). 

\section{Problem Definition: Guarding Label-Only Access to Networks}
\label{sec:IDP} 
In this section, we define the problem of protecting label-only access to neural network classifiers using individual differential privacy. 
We begin by defining the functions that may leak private information in a label-only setting. We then define the problem of creating a randomized mechanism for them that introduces a minor accuracy decrease. Lastly, we describe two naive approaches. 


\subsection{Label-Only Queries} 

In this section, we formalize the functions that access the dataset (in our setting, it is also called the training set), thus may leak private information, in our label-only access setting. 

Label-only access to a neural network $N$ maps an input $x\in [0,1]^d$ to a label $c\in C$, where  
$N$ is the network trained by a training algorithm $\mathcal{T}$ given a dataset $D$ and an architecture 
$\widetilde{N}$.
One might consider the training algorithm as the function that leaks private information, since it accesses the dataset to compute the network. 
However, this definition is too coarse for our label-only access setting, where 
users (including attackers) indirectly access the dataset by obtaining the network's predicted label of any input and do not have access to the network's internal parameters.
 Further, making the training algorithm private requires adding random noise to its computations, regardless of the inputs that are later used for querying the network.  
Our formalization removes this separation of the network's training and its querying. It defines the functions that access the dataset as an infinite set of functions, one for each possible input $x\in [0,1]^d$ (not only inputs in the dataset). This refined definition enables us later to add noise only to the functions that potentially leak private information. 
Formally, given a training algorithm $\mathcal{T}$ and an architecture $\widetilde{N}$,
the set of \emph{label-only queries} is: $\mathcal{F}_{\mathcal{T},\widetilde{N}} = \{f_{x,\mathcal{T},\widetilde{N}} \mid x\in [0,1]^d\}$, where for every $x\in [0,1]^d$ the function $f_{x,\mathcal{T},\widetilde{N}}(D)$ maps a dataset $D$ to the label predicted for $x$ by the classifier trained by $\mathcal{T}$ on $\widetilde{N}$ and $D$. 
To illustrate, consider stochastic gradient descent as the training algorithm $\mathcal{T}_\text{SGD}$ and as a network architecture $\widetilde{N}_{3,5,2}$ a fully-connected network with three layers, consisting of three neurons, five neurons, and two neurons, respectively. 
Then, for example the function $f_{(1,0.2,0.1)}\in \mathcal{F}_{\mathcal{T}_\text{SGD},\widetilde{N}_{3,5,2}}$ maps a dataset 
$D$ to one of the two labels by (1)~training the network $N=\mathcal{T}_\text{SGD}(D,\widetilde{N}_{3,5,2})$, (2)~computing $N((1,0.2,0.1))$ and (3)~returning the label with the maximal score $c=\text{argmax}(N(1,0.2,0.1))$.

Our definition may raise two questions: (1)~is it equivalent to the common practice where a neural network is computed once and then queried? and (2)~is it useful in practice if it (supposedly) requires to retrain the network from scratch upon every new input? 
The answer to both questions is yes. First, it is equivalent because given a dataset $D$ and an  architecture $\widetilde{N}$,
  the training algorithm $\mathcal{T}$ is \emph{deterministic}. We note that if a training algorithm $\mathcal{T}$ makes random choices, we assume we are given its random seed, which makes $\mathcal{T}$ a deterministic function. Namely, every function $f_{x,\mathcal{T},\widetilde{N}}(D)$ computes the exact same network given $\mathcal{T}$, $D$, and $\widetilde{N}$ and thus our definition is equivalent to the common practice. 
  Second, in practice, we do not retrain the network from scratch for every input, we define these functions only for the sake of mathematically characterizing the functions that may leak private information, but eventually the network is trained once (as we explain later). 

\subsection{Problem Definition: \propi Label-Only Queries with a Minor Accuracy Decrease}
In this section, we present our problem of creating a randomized mechanism for our label-only queries, while minimally decreasing their accuracy. 

As explained in~\Cref{sec:preliminary}, creating a randomized mechanism for our label-only queries is possible by adding noise determined by the local sensitivity. We next define the local sensitivity of a label-only query. 
As defined in~\Cref{def:local_sensitivity}, given a function $f_{x,\mathcal{T},\widetilde{N}}$
and a dataset $D$,  
the local sensitivity is the maximum difference of $f_{x,\mathcal{T},\widetilde{N}}$ with respect to $D$ and any of its adjacent datasets. 
Since a label-only query returns a label, we define the difference over their one-hot encoding. That is, if $f_{x,\mathcal{T},\widetilde{N}}(D)=c$, then the one-hot encoding is a vector of dimension $|C|$ where the entry of $c$ is one and the other entries are zeros. By this definition, the local sensitivity is zero if for \emph{every} adjacent dataset $D'$, the same label is returned: $f_{x,\mathcal{T},\widetilde{N}}(D)=f_{x,\mathcal{T},\widetilde{N}}(D')$. 
In this case, $f_{x,\mathcal{T},\widetilde{N}}$ is 0-\propi.
It is not zero if there exists an adjacent dataset $D'$ that returns a different label: $f_{x,\mathcal{T},\widetilde{N}}(D)\neq f_{x,\mathcal{T},\widetilde{N}}(D')$.

To illustrate the label-only queries with sensitivity zero, consider \Cref{fig::intro}(a).
As described, in this example, we focus on a binary network classifier trained on a 2D synthetic training set of size 7,000. 
The figure shows the decision boundaries of this classifier along with the decision boundaries of the networks trained on its adjacent datasets. The figure shows that the decision boundaries are not identical but close. Thus, many inputs are classified the same by all classifiers (shown in light blue background). The respective label-only queries of these inputs have local sensitivity of zero. The other inputs are classified differently by at least one network trained on an adjacent dataset. Namely, the local sensitivity of their respective label-only queries is not zero. 

While we could design an \propi randomized version of the label-only queries by adding noise that assumes the maximum local sensitivity, it would add unnecessary noise to queries with local sensitivity zero and unnecessarily decrease their accuracy. 
For example, in \Cref{fig::intro}(a) the local sensitivity of most label-only queries is zero, and thus these queries are 0-\propi without any noise addition. By adding noise only to queries with non-zero local sensitivity, we can obtain an \propi label-only access to the network with overall a minor decrease in its accuracy.
This is our problem:


\begin{definition}[Problem Definition: Minimal Noise \propi Label-Only Queries]~\label{def:NN_local_sensitivity}
Given a dataset $D$, a training algorithm $\mathcal{T}$, a network architecture $\widetilde{N}$, and a privacy budget $\varepsilon$, our goal is to compute for every label-only query $f_{x,\mathcal{T},\widetilde{N}}$, for $x\in [0,1]^d$, the minimal noise required to make it $\varepsilon$-\propi. 
\end{definition} 
Note that we have access to $D$, $\mathcal{T}$, and $\widetilde{N}$; only the users have label-only access to the trained network (as common in machine-learning-as-a-service platforms).
In the following, for brevity,  
we say that an input $x$ satisfies/violates \propi if its label-only query $f_{x,\mathcal{T},\widetilde{N}}(D)$ satisfies/violates \propi. Similarly, the local sensitivity of $x$ is the local sensitivity of its label-only query $f_{x,\mathcal{T},\widetilde{N}}(D)$.
 We call an input $x$ whose local sensitivity is not zero a \emph{leaking input}, since its label-only query is not 0-\propi. 


\subsection{Naive Algorithms} \label{sec:naive}
In this section, we describe two naive algorithms that solve our problem but are inefficient: either in terms of accuracy or in terms of time and space. 

\paragraph{The naive noise algorithm} The first naive algorithm invokes the training algorithm $\mathcal{T}$ on the given dataset $D$ and the network architecture $\widetilde{N}$ to compute $N$. Its label-only access is defined as follows. Given an input $x\in [0,1]^d$,
 it computes $f_{x,\mathcal{T},\widetilde{N}}(D)$
 by passing $x$ through $N$. Then, it ensures $\varepsilon$-\propi by employing the \emph{exponential mechanism} with the maximum local sensitivity. 
We next provide background on the exponential mechanism and explain how this algorithm uses it. 

The exponential mechanism is a popular approach to provide DP guarantees for non-numerical algorithms, such as classification algorithms~\citep{ref_91,Dwork06,ref_92}. In these scenarios, typically, the output range is discrete and denoted $\mathcal{R}$.  
The exponential mechanism defines the \emph{global} sensitivity with respect to a given
\emph{utility} function $u:\mathcal{D} \times \mathcal{R} \rightarrow \mathbb{R}$, 
mapping a dataset and an output to a score.
The global sensitivity is: $\Delta u = \max_{r \in \mathcal{R}} \max_{D,D' s.t. ||D - D'|| \leq 1} |u(D,r) - u(D',r)|$.
To ensure $\varepsilon$-DP, the exponential mechanism returns an output $r \in \mathcal{R}$, 
with a probability proportional to $\exp(\varepsilon \cdot u(D,r) / (2\Delta u))$ (we refer to this process as adding noise to the output $r$). 

\sloppy
Given a query 
$f_{x,\mathcal{T},\widetilde{N}}: \mathcal{D} \rightarrow C$ (mapping a dataset to a class), 
we define a utility function $u_{x,\mathcal{T},\widetilde{N}} : \mathcal{D} \times C \rightarrow \{0,1\}$ (we sometimes abbreviate by $u$ to simplify notation).
The utility function $u$ maps a dataset $D$ and a class $c$ to $1$ if $f_{x,\mathcal{T},\widetilde{N}}(D)=c$ and to $0$ otherwise.
For example, if $f_{x,\mathcal{T},\widetilde{N}}(D)=c_0$, then  $u_{x,\mathcal{T},\widetilde{N}}(D,c_0)=1$ and   $u_{x,\mathcal{T},\widetilde{N}}(D,c)=0$, for $c\in C\setminus\{c_0\}$.
Given a dataset $D$, the local sensitivity of $f_{x,\mathcal{T},\widetilde{N}}$ is $\Delta u = \max_{c \in C} \max_{D' s.t. ||D - D'|| \leq 1} |u(D,c) - u(D',c)|$. 
Note that $\Delta u\in\{0,1\}$. 
If our algorithms identify $\Delta u=0$, they do not invoke the exponential mechanism.
That is, they always invoke it with $\Delta u=1$. Thus, our exponential mechanism
returns a class $c \in C$ with a probability proportional to $\exp(\varepsilon\cdot u(D,c) / 2)$.  
Our next lemma states that this mechanism guarantees $\varepsilon$-\propi:

\begin{restatable}[]{lemma}{ftb}
\label{lemma::exponential_mechanism}
Given a dataset $D$, a query $f_{x,\mathcal{T},\widetilde{N}}$, and a privacy budget $\varepsilon$, the exponential mechanism with our utility function $u_{x,\mathcal{T},\widetilde{N}}$ 
is $\varepsilon$-\propi.
\end{restatable}
The proof \ifthenelse{\EXTENDEDVER<0}{(\Cref{sec:proofs})}{\cite[Appendix A]{ref_105}} is similar to the proof showing that this mechanism is DP (with respect to the global sensitivity).
  This lemma implies that the naive noise algorithm solves \Cref{def:NN_local_sensitivity}.
However, it introduces a very high noise since for any input it invokes the exponential mechanism with the maximum local sensitivity. As we show in~\Cref{sec:eval}, its accuracy decrease is even higher than DP training algorithms.
In fact, in this case, the maximum local sensitivity is equal to the maximum global sensitivity, thus the naive noise algorithm is DP.

\paragraph{The naive \propi algorithm} This naive algorithm improves the accuracy of the previous algorithm by identifying the local sensitivity of each query, instead of bounding it by the maximum local sensitivity.  
It begins by 
training the set of classifiers for $D$ and all its adjacent datasets.
That is, it computes $\mathcal{N}=\{N\}\cup \{N_{-x_D}\mid x_D\in D \}$, where $N=\mathcal{T}(D,\widetilde{N}) $ and $N_{-x_D}=\mathcal{T}(D\setminus\{x_D\},\widetilde{N}) $.
Its label-only access is defined as follows. Given an input $x\in [0,1]^d$, to compute $f_{x,\mathcal{T},\widetilde{N}}(D)$ it passes $x$ through all networks in $\mathcal{N}$. If all networks return the same label, it returns the label as is. 
Otherwise, it employs the exponential mechanism, exactly as the naive noise algorithm does. 
Namely, it employs the exponential mechanism only for inputs whose local sensitivity is not zero and thus  
obtains $\varepsilon$-\propi with minimum accuracy decrease.
However, as we show in~\Cref{sec:eval}, its inference time is high, 
because every input passes through $|D|+1$ classifiers.
Also it poses a significant space overhead, because it requires to store in the memory all $|D|+1$ classifiers throughout its execution.


\section{Insight: The Individual Differential Privacy Deterministic Bound}
\label{sec:formalization}
In this section, we present our main insight:
identifying inputs with local sensitivity zero using the  
\emph{network's \propl} (\propa). 
The \propa overapproximates the leaking inputs (whose local sensitivity is not zero) using the network's classification confidence.
Every input that the network classifies with confidence over this bound satisfies 0-\propi. 
We next provide the definitions, illustrate the \propa, and discuss the challenges in computing it. 

\paragraph{Leaking inputs}
The set of leaking inputs can be defined by the decision boundaries of all classifiers in $\{N\}\cup \{N_{-x_D}\mid x_D \in D\}$, where $N$ is the classifier trained on the dataset $D$ and $N_{-x_D}$ is the classifier trained on $D\setminus\{x_D\}$.  
The decision boundaries of a classifier partition the input space into subspaces of inputs classified to the same class. 
The leaking inputs are the inputs that for $N$ are contained in a subspace labeled as some class and 
for some $N_{-x_D}$, where $x_D \in D$, are contained in a subspace labeled as another class.
We formalize this definition using \emph{classification confidences}. 
%
Given a classifier $N$, an input $x$, and a label $c$,
the classification confidence is the difference between the score of $c$ and the highest score of the other classes: 
%
%
$\mathcal{C}_{N}^c(x) = N(x)_c - \max_{c' \neq c} N(x)_{c'}.$
%
If $\mathcal{C}_N^c(x)$ is positive, then $N$ classifies $x$ as $c$.
Otherwise, $N$ does not classify $x$ as $c$. 
Given this notation, the set of leaking inputs of class $c$ contains all inputs classified by $N$ as $c$ 
and by one of the other classifiers not as $c$:
$\{x\in [0,1]^d \mid \mathcal{C}_{N}^c(x) > 0 \land \bigvee_{x_D \in D} \mathcal{C}_{N_{-x_D}}^c(x) \leq 0\}$. 
Computing this set requires computing the decision boundaries, which is highly complex. 
Instead, we overapproximate this set. 

\paragraph{The \propa}
 We overapproximate the set of leaking inputs by the maximal classification confidence $\beta^*$ of any leaking input. That is, any input whose classification confidence is higher than $\beta^*$ satisfies 0-\propi \emph{without added noise}. Our overapproximation is sound: it does not miss any leaking input. However, it may be imprecise: there may be inputs that are not leaking whose classification confidence is smaller or equal to $\beta^*$. Importantly, adding noise based on our overapproximation enables \tool to satisfy \propi: every leaking input will be noised. Due to the 
  overapproximation, \tool may add noise to inputs that are not leaking, however, as we show in~\Cref{sec:eval}, this approach enables \tool to obtain label-only access to a network with a small accuracy decrease. 
    We call $\beta^*$ the \emph{\propl}, since inputs whose confidence is above $\beta^*$  deterministically satisfy \propi, without adding random noise.
We next formally define it.


\begin{definition}[Individual Differential Privacy Deterministic Bound (\propa)]~\label{def:main_problem}
Given a dataset~$D$, a training algorithm $\mathcal{T}$, and a network architecture $\widetilde{N}$, 
the \emph{\propa} is: 
$$
\beta^*=\text{argmax}_\beta \exists x \in [0,1]^d\ \exists c\in C \ 
\left( 
\mathcal{C}_{N}^c(x) \geq \beta \land \bigvee_{ x_D \in D}\ \mathcal{C}_{N_{-x_D}}^c(x) \leq 0\right )
$$
where $N=\mathcal{T}(D,\widetilde{N})$ and $N_{-x_D}=\mathcal{T}(D\setminus \{x_D\},\widetilde{N})$, for $x_D\in D$.
\end{definition}

We note the following.
First, the \propa exists for every network, since there is a maximum confidence for every label~$c$. In the worst case scenario, this bound is equal to the maximum confidence over all labels, in which case the set of leaking inputs is overapproximated by all inputs. Consequently, every input will be noised by \tool.
In the (highly unlikely) best case scenario, this bound is zero, indicating that all classifiers classify the same every input. 
Second, every higher classification confidence $\beta>\beta^*$ also overapproximates the set of leaking inputs. However, it adds a higher than necessary overapproximation error, which would lead \tool to add unnecessary noise.

To reduce the overapproximation error, we define the \propa for every class $c\in C$:
\begin{align}\label{dpdbc}
\beta^*_c=\text{argmax}_\beta \exists x \in [0,1]^d
\left( 
\mathcal{C}_{N}^c(x) \geq \beta \land \bigvee_{ x_D \in D}\ \mathcal{C}_{N_{-x_D}}^c(x) \leq 0\right )
\end{align}
Computing the \propa for every class $c$ enables \tool to lower the decrease in the network's accuracy: a joint bound for all classes $\beta^*=\max_c \beta^*_c$ would lead \tool to add unnecessary noise to inputs that the network classifies as $c$ with confidence 
in $(\beta_c^*,\beta^*]$.

\paragraph{Illustration}
We illustrate the \propa on the classifier considered in \Cref{fig::intro}(a),
trained on a 2D synthetic training set comprising 7,000 data points. A labeled data point $(x_D,y_D)$ is $x_D=(x_1,x_2)\in [0,1]^2$ and $y_D\in \{0,1\}$. 
Its \propa-s are $\beta_0^*=6.6$ and $\beta_1^*=15$ and thus $\beta^*=15$.
Our idea is to make label-only access to a network $N$ \propi by adding noise to every input $x$ whose confidence is at most $\beta^*_c$, where $c$ is the class $N$ predicts for $x$.
To visualize the set of inputs overapproximated by the \propa (which will be noised by \tool), we consider an extended definition of the classification confidence: $\mathcal{C}^{c,\beta}_{N}(x) = \max(0, \mathcal{C}_{N}^c(x) - \beta)$. That is, classification confidence up to $\beta$ is set to zero so inputs classified with such confidence are on the extended decision boundary. 
\Cref{fig::formalization_beta} visualizes the extended classification confidence for $\beta\in\{0, \beta_0^*, \beta_1^*, 30\}$. For each $\beta$, it shows the extended classification confidence as a function of the input $(x_1,x_2)$, along with the decision boundaries of the 7,000 classifiers in $\{N_{-x_D}\mid x_D \in D \}$ (in colored lines). For $\beta=0$, the decision boundary of $N$ does not cover all the other classifiers' decision boundaries. Namely, there are leaking inputs not on the extended boundary and thus $\beta=0$ is an underapproximation of the set of leaking inputs. Adding noise only to these inputs will not make the label-only access to the network \propi. For $\beta=\beta_0^*$, all the leaking inputs classified as $0$ are on the extended decision boundary, but there are leaking inputs classified as $1$ not on it. For $\beta_1^*$, the extended decision boundary covers all decision boundaries (i.e., it covers all classification differences) and 
thus $\beta^*_1$ overapproximates the set of leaking inputs.
It is also the minimal bound covering all decision boundaries. Adding noise to these inputs will make the label-only access to the network \propi (in fact, for inputs that the network classifies as $0$, adding noise is required only if their confidence is at most $\beta_0^*$). For $\beta=30$, all decision boundaries are covered by the extended decision boundary and it also provides an overapproximation of the leaking inputs. While adding noise to all these inputs will make the label-only access to the network \propi, it results in an unnecessary decrease in accuracy.
 \begin{figure}[t]
    \centering
  \includegraphics[width=1\linewidth, trim=0 295 0 0, clip,page=4]{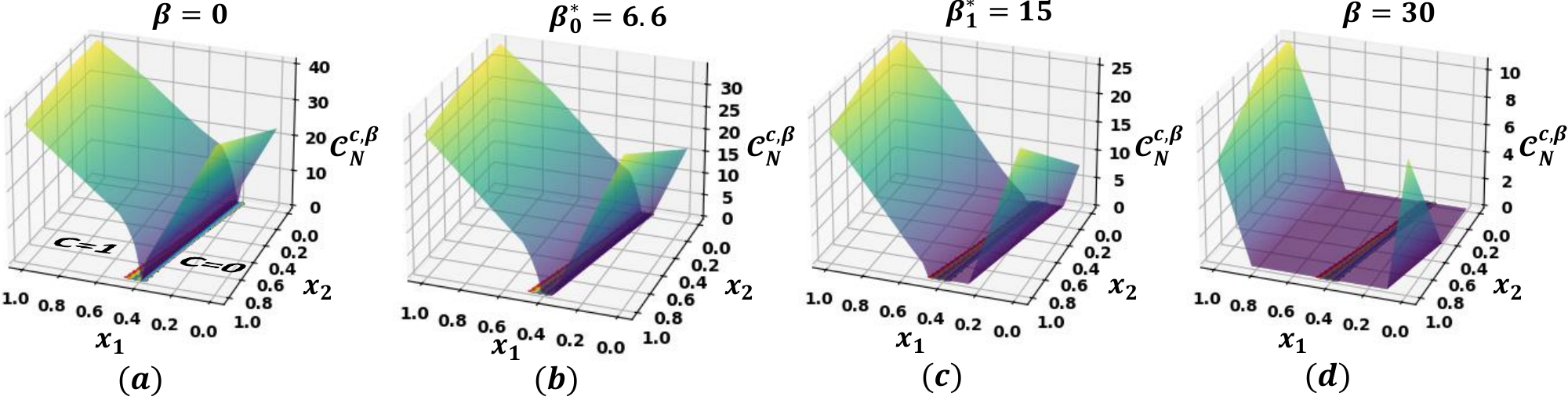}
    \caption{The conceptually extended decision boundaries as a function of the input $(x_1,x_2)\in[0,1]^2$, for different confidences, including the \propa.}
    \label{fig::formalization_beta}
\end{figure}  
  %
%
Conceptually, our \propa 
partitions the input space into two parts: a subspace of inputs that deterministically satisfy \propi (if the attacker only queries them, the label-only access is $0$-\propi) and a subspace 
that overapproximates the leaking inputs. 
To ensure \propi in this subspace, we employ the exponential mechanism.
\sloppy
\paragraph{Challenges} Computing the \propa is very challenging. First, it is a global property, requiring to analyze the network's computation for \emph{every} input (not confined to a particular dataset). 
This analysis has shown to be highly complex~\citep{ref_51,ref_52}, as it involves precise modeling of a network classifier (a highly non-convex function) over a very large space of inputs. 
Second, \propa is defined over a very large number of classifiers, $|D|+1$, where $D$ is the dataset, which tends to be large (several thousand). This amplifies the complexity of the previous challenge. 
Third, \propa is the \emph{maximal} bound $\beta$ satisfying 
$\exists x \in [0,1]^d\ \exists c\in C \ \left( 
\mathcal{C}_{N}^c(x) \geq \beta \land \bigvee x_D \in D.\ \mathcal{C}_{N_{-x_D}}^c(x) \leq 0\right )$, where $\beta$ is a real number, which 
increases the problem's complexity. 

\section{Overview on Computing the \propa}
\label{sec:overview}
In this section, we present our ideas to efficiently compute the \propa. 
First, we formalize \propa as a constrained problem (MILP), which can be solved by existing solvers 
but its complexity is very high. 
To cope, we abstract the analyzed networks using a \emph{hyper-network}. To mitigate the abstraction's overapproximation error, we introduce a branch-and-bound technique that refines a hyper-network into multiple hyper-networks, each abstracts a disjoint set of networks. 
To further reduce the complexity, we bound the differences of matching neurons in a network and a hyper-network and 
add them as linear constraints or, if they are small, employ linear relaxation.

\subsection{Encoding \propa as a MILP}
\label{sec:overview_opt}
We express the problem of computing the \propa, which is a constrained optimization, as a mixed-integer linear program (MILP).    
MILP has been employed by many verifiers for solving constrained optimizations, e.g.,  
local robustness~\cite{ref_86,ref_49,ref_42}, global robustness properties~\cite{ref_5,ref_6,ref_7}, and privacy in local neighborhoods~\cite{ref_8}.
Compared to these verifiers, computing \propa is much more complex: it is a global property (pertaining to any input) \emph{and} over a very large number of networks (the size of the training set plus one), which is the reason we rely on additional ideas to scale the computation.
An advantage of existing MILP optimizers is that they are anytime algorithms. That is, the optimizer can return at any point an interval bounding the value of the optimal solution (given enough time, this interval contains only the optimal value).

Recall that for $c\in C$, the \propa is the maximal $\beta_c$ satisfying $\mathcal{C}_{N}^c(x) \geq \beta_c \land \bigvee_{ x_D \in D}\ \mathcal{C}_{N_{-x_D}}^c(x) \leq 0$ (\Cref{dpdbc}). 
 A straightforward MILP leverages prior work for encoding the classification confidence~\cite{ref_42} and
encodes the disjunction (i.e., $\bigvee$) 
using the Big M method~\cite{ref_73,ref_75}.
However, this MILP has an exponential complexity in $|D|$ and the multiplication of
the number of non-input neurons in $N$ and $|D|+1$. 
This is because a MILP's complexity is exponential in the number of boolean variables and this encoding introduces $|D|$ boolean variables for the disjunction and a unique boolean variable for every non-input neuron in every network.

A naive approach to cope with this complexity is to express the \propa as $|D|$ separate MILPs $\{P_{x_D} \mid x_D \in D \}$, where $P_{x_D}$ is
$\beta_{c,\{x_D\}}^*=\text{argmax}_{\beta} \exists x 
\left(\mathcal{C}_{N}^c(x) \geq \beta \land \mathcal{C}_{N_{-{x_D}}}^c(x) \leq 0\right )$.
Each MILP is submitted to the solver and the \propa is the maximum: $\beta_c^* = \max_{x_D\in D}\beta^*_{c,\{x_D\}}$. 
The complexity of this naive approach is $|D|$ times the complexity of each MILP $P_{x_D}$, which is exponential in the multiplication of the number of non-input neurons in $N$ and two.
%
However, the naive approach is impractical because it requires solving a large number of MILPs. Additionally, since it requires the bound of all $|D|$ problems, 
obtaining an anytime solution requires obtaining an anytime solution to all $|D|$ problems, which is computationally expensive. 

\subsection{Hyper-Networks}\label{sec:overview_hyper}
To define a MILP with a lower complexity, we rely on \emph{hyper-networks}.
%
A hyper-network abstracts a set of networks by associating the network parameters' \emph{intervals}, defined by the minimum and maximum values over all networks. Formally:

\begin{definition}[A Hyper-Network]\label{over:Hyper}
Given a set of networks $\mathcal{N}=\{N_n \mid n\in [K]\}$ of the same architecture, each with weights $\mathcal{W}_{n} = \{w^n_{1,1,1},\ldots,w^n_{L,k_L,k_{L-1}}\}$ and biases $\mathcal{B}_{n}=\{b^n_{1,1},\ldots,b^n_{L,k_L}\}$, a \emph{hyper-network} $N^\#$ is a network of the same architecture such that $w_{m,k,k'}^\# = [\min_{n\in[K]}(w^n_{m,k,k'}), \max_{n\in[K]}(w^n_{m,k,k'})]$ and $b_{m,k}^\#=[\min_{n\in[K]}(b^n_{m,k}),\max_{n\in[K]}(b^n_{m,k})]$.
\end{definition}

\Cref{fig::hyper_bab} shows an example of four networks $N_{-1},\ldots,N_{-4}$ and their hyper-network $N^\#_{1,2,3,4}$. For example, since the values of $z_{1,1}$'s bias in the networks are $0.1$, $0$, $0.1$, $0.3$, its interval in the hyper-network is $[0,0.3]$.
A hyper-network introduces an overapproximation error. This is because it abstracts every network whose parameters are contained in their respective hyper-network's intervals. Formally, a hyper-network $N^\#$ abstracts every network $N'$ whose weights satisfy $w_{m,k,k'}\in w_{m,k,k'}^\#$ and biases satisfy $b_{m,k}\in b_{m,k}^\#$. For example, the hyper-network $N^\#_{1,2,3,4}$ in~\Cref{fig::hyper_bab} abstracts the network that is identical to $N_{-1}$ except that the bias of $z_{1,1}$ is $0.2$, even though this network is not one of the four networks used to define this hyper-network. 
 We note that hyper-networks have been proposed by~\citet{ref_8}, however, their focus is different: they predict a hyper-network from a subset of networks.
 
 Given a subset of the dataset $S\subseteq D$, we define the set of networks of $S$ as $\mathcal{N}_S=\{N_{-x_D}\mid x_D \in S \}$ and denote the hyper-network abstracting the set of networks $\mathcal{N}_S$ by $N^\#_S$.
The \propa of $S$ is defined as the following constrained optimization over its hyper-network $N^\#_{S}$:
\begin{align}\label{dphyp}
\beta^*_{c,S}=\text{argmax}_\beta \exists x 
\left( 
\mathcal{C}_{N}^c(x) \geq \beta \land \mathcal{C}_{{N}_{S}^\#}^{c}(x) \leq 0\right )
\end{align}
Similarly to~\Cref{sec:overview_opt}, we can encode this problem as a MILP, where the neurons' weighted sums depend on the hyper-network's intervals.
By solving this MILP for $S=D$,
we can overapproximate the true \propa (\Cref{dpdbc}) by a greater or equal bound.
This follows since $\beta^*_c= \max_{x_D \in D }\beta^*_{c,\{x_D\}}$ and $\forall x_D \in D.\ \beta_{c,\{x_D\}}^*\leq \beta^*_{c,D}$. 
This MILP's complexity is exponential in $2\cdot |N|$, where $|N|$ is the number of non-input neurons.
This complexity is significantly lower than the complexity of the approaches of \Cref{sec:overview_opt}. 
Also, by expressing the problem as a single MILP (unlike the naive approach), we can benefit from the solver's anytime solution.
However, the overapproximation error is high: 
$\beta^*_{c,D}$ is significantly larger than the true \propa, because it considers \emph{every} network abstracted by the hyper-network, even networks that are not used to define the hyper-network.
Having a significantly larger bound than the true \propa would lead \tool to add much more noise than necessary to make the label-only access to the network \propi, and thus significantly decrease its accuracy. We next explain how to compute the true \propa with hyper-networks.

\begin{figure}[t]
    \centering
  \includegraphics[width=1\linewidth, trim=0 135 0 0, clip,page=6]{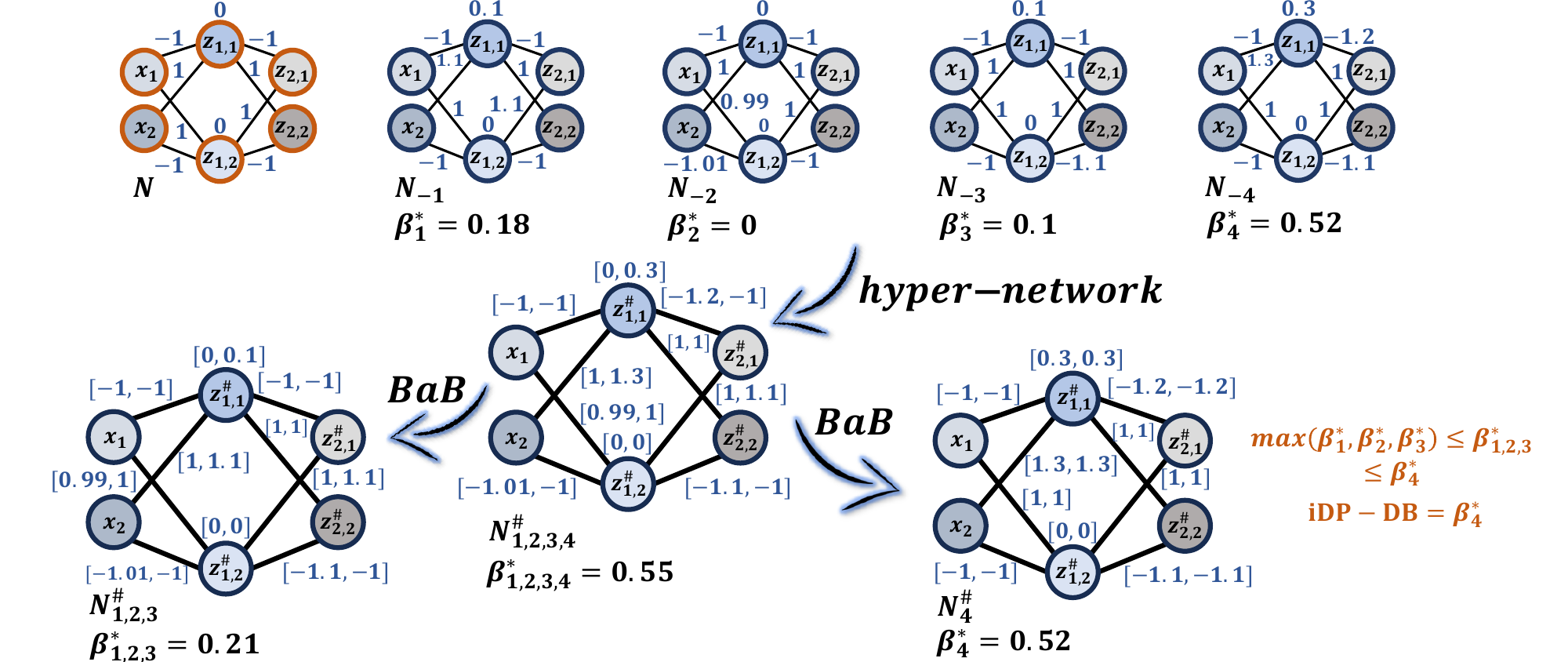}
    \caption{An example of a hyper-network and our branch-and-bound.}
    \label{fig::hyper_bab}
\end{figure}

\subsection{A Branch-and-Bound for Hyper-Networks}\label{sec:overview_bab}
In this section, we present a novel branch-and-bound technique that enables us to compute the true \propa using hyper-networks.

At a high-level, a branch-and-bound (BaB) technique is defined over a verification analysis of abstract objects and it eliminates the analysis' overapproximation error. Given an abstract object, it runs the analysis, which returns a value. If this value is smaller or equal to a certain bound (the object is \emph{bounded}), it terminates. If this value is greater than the bound, the object is refined to several abstract objects (the object is \emph{branched}) and this operation continues for each abstract object. BaB terminates when there are no more objects to branch. Its main advantages are: (1)~it attempts to reduce the overall execution time of the analysis by analyzing abstract objects instead of analyzing independently a large (or even infinite) set of concrete objects and (2)~it does not lose precision despite of analyzing abstract objects, since if the analysis returns a value above the bound, the abstract object is refined.
Thanks to these advantages, several BaB techniques have been proposed for local robustness verification~\cite{ref76,ref77,ref78,ref79,ref80,ref81}. 
However, determining local robustness is simpler than computing the \propa, since it requires analyzing the predicted labels of a single network for a set of inputs. In contrast, computing the \propa requires analyzing the predicted labels of a \emph{large set of networks} for \emph{any input}. 

We propose a new BaB technique to compute the precise \propa. Its verification analysis takes as input a hyper-network abstracting the set of networks of a set $S$ and it computes the \propa of $S$ (\Cref{dphyp}).
%
For branching, our BaB refines a hyper-network $N^\#_S$ into $K$ hyper-networks $\{N_{S_1}^\#,\ldots,N_{S_K}^\#\}$, such that the sets $S_1, \ldots,S_K$ partition $S$.
Our BaB defines the partitioning by clustering the networks in $\mathcal{N}_S$ based on their parameters' similarity. 
The motivation for this clustering is that classifiers whose parameters have closer values are likely to have closer \propa. Consequently, their hyper-network's \propa is likely to have a lower overapproximation error, which increases the likelihood to succeed bounding it and avoid further branching.  

For bounding, we rely on the following observation. Given a set $S$, if its \propa $\beta^*_{c,S}$ is \emph{smaller or equal to} $\beta^*_{c,\{\hat{x}_D\}}$ for some $\hat{x}_D \in D$, then $N^\#_S$ need not be branched. This is because the true \propa is $\beta_c^* = \max_{x_D \in D}\beta^*_{c,\{x_D\}}$ and $\forall x_D'\in S.\ \beta^*_{c,\{x_D'\}}\leq \beta^*_{c,S}$. By transitivity,  $\forall x_D'\in S.\ \beta^*_{c,\{x_D'\}}\leq \beta^*_{c,\{\hat{x}_D\}}\leq \beta^*_{c}$. Namely, our BaB bounds hyper-networks by the \propa of hyper-networks that abstract a single network (in which case there is no overapproximation error). 


Our BaB has several advantages. First, it computes the precise \propa. 
Second, it can provide an anytime solution: if it early stops, it returns an overapproximation of the \propa. Third, it relies on MILPs over hyper-networks, which reduces the number of boolean variables compared to the straightforward MILP encoding. Fourth, it dynamically identifies a minimal number of MILPs for computing the \propa. Fifth, it orders the analysis of the hyper-networks to increase the chances of bounding them (described in~\Cref{sec:ourapp_sys}).


\Cref{fig::hyper_bab} exemplifies our BaB for computing the \propa of $c=0$ for a classifier $N$ (top 
left) 
 and four networks $\{N_{-1}, N_{-2}, N_{-3}, N_{-4}\}$ (top). 
To simplify notation, we omit the subscript $c=0$ from the \propa.
The naive approach (\Cref{sec:overview_opt}) solves four optimization problems $P_{x_D}$ for $x_D \in [4]$ and returns their maximum (obtained for $P_4$ in this example): $\beta^* = \beta^*_4=0.52$. 
In contrast, our BaB obtains the same bound by solving only three MILPs of the same complexity as $P_{x_D}$. Our BaB begins by computing the \propa of $\{1,2,3,4\}$, which is $\beta^*_{1,\ldots,4}=0.55$ (bottom center).
Although we know upfront that this \propa is not tight, our BaB computes it to provide an anytime solution. 
Since this hyper-network is not bounded by any \propa of a single network, it is branched.  
To this end, our BaB clusters the networks based on their similarity. This results in the partitioning $\mathcal{N}_{1,2,3}$ and $\mathcal{N}_{4}$. Accordingly, two hyper-networks are constructed (bottom left and right). 
For each, our BaB solves the respective MILP, returning $\beta^*_{1,\ldots,3}=0.21$ and $\beta^*_{4} =0.52$. Since 
 $\beta^*_{1,\ldots,3}\leq\beta^*_4$, it bounds the hyper-network ${N}^\#_{1,2,3}$. Similarly, it bounds ${N}^\#_4$ by $\beta^*_{4} $.  
Then, our BaB terminates and returns $0.52$.

\subsection{Matching Dependencies and Relax-If-Similar}\label{sec:overview_match}
In this section, we describe two techniques to further reduce the complexity of the MILP of~\Cref{dphyp}. 
Both techniques rely on computing bounds for the differences of matching neurons in the network and the hyper-network.
The first technique, \emph{matching dependencies}, encodes these differences as linear constraints, for every pair of matching neurons, to prune the search space. The second technique, \emph{relax-if-similar}, overapproximates neurons in the hyper-network whose difference is very small by linear relaxation. It reduces the complexity's exponent by one for every overapproximated neuron. While overapproximation reduces the precision, it is employed when the difference is small and combined with the matching dependencies, the precision loss is small. 
  

These techniques rely on bounding the differences of matching neurons. A pair of matching neurons consists of a neuron $z_{m,k}$ in the network $N$ and its corresponding neuron in the hyper-network $N^\#$, denoted $z_{m,k}^\#$, whose output is a real-valued interval (since the weights and biases of a hyper-network are intervals).
We bound the difference of $z^\#_{m,k}$ and $z_{m,k}$ in an interval $[\Delta^l_{m,k},\Delta^u_{m,k}]\in \mathbb{R}^2$ overapproximating the expression $z^\#_{m,k}-[z_{m,k},z_{m,k}]$.
The difference intervals enable significant pruning to our search space, since in our setting the outputs of matching neurons are highly dependent. This is because the network and hyper-network accept the same input and since their respective weights and biases are very close, because $N$ and the networks used to define $N^\#$ are trained by the same training algorithm and their training sets only slightly differ. 
To compute the difference intervals, we rewrite the interval of every weight (and bias) in the hyper-network $[\underline{w^\#}_{m,k,k'},\overline{w^\#}_{m,k,k'}]$ in terms of the weight in $N$ plus a difference: $[w_{m,k,k'}- (w_{m,k,k'}-\underline{w^\#}_{m,k,k'}), w_{m,k,k'}+{(\overline{w^\#}_{m,k,k'}-w_{m,k,k'})}]$. 
Then, we employ bound propagation 
(formalized in~\Cref{sec:our_approach_milp}) to overapproximate $z^\#_{m,k}-[z_{m,k},z_{m,k}]$, for all $m$ and $k$. 

The matching dependencies technique adds to 
\Cref{dphyp} the difference intervals.
For every non-input neuron $z_{m,k}$ and its difference interval $[\Delta^l_{m,k},\Delta^u_{m,k}]$, it adds the constraint: 
$z_{m,k} + \Delta^l_{m,k} \leq z^\#_{m,k} \leq z_{m,k} + \Delta^u_{m,k}$.
While these constraints can be added to any pair of networks, they are more effective for pruning the search space when $\Delta^u_{m,k}-\Delta^l_{m,k}$ is small, which is the case in our setting.
 
The relax-if-similar technique overapproximates the computation of neurons in the hyper-network if their difference interval is small. 
We remind that our MILP encoding of~\Cref{dphyp} leverages prior work for encoding the classification confidence~\cite{ref_42}.
This encoding introduces a unique boolean variable for every non-input neuron in the network and the hyper-network.
As described, the complexity of a MILP is exponential in the number of boolean variables. 
 To reduce the exponential complexity, prior work~\cite{ref_49,ref_7,ref_5,ref_6,ref_50} eliminates boolean variables by: 
 (1)~computing tight lower and upper bounds to identify neurons whose output is non-positive or non-negative, in which case their ReLU is stable and their boolean variable can be removed, and/or 
(2)~overapproximating the ReLU computations using linear constraints without boolean variables. 
Although in general overapproximation leads to precision loss, we observe that if the difference interval of a neuron $z_{m,k}^\#$ is small, we can overapproximate its computation without losing too much precision. This follows because its matching neuron $z_{m,k}$ is precisely encoded (with a boolean variable) and because the matching dependency of $z_{m,k}$ forces the value of $z^\#_{m,k}$ to remain close to $z_{m,k}$. 
Based on this insight, our relax-if-similar technique 
eliminates the boolean variable of $z_{m,k}^\#$ and replaces its ReLU constraints by linear relaxation constraints.
These constraints capture the minimal triangle bounding
the piecewise linear function of ReLU 
using three linear constraints~\cite{Ehlers17}.

\Cref{fig::closness_mutual_encoding} exemplifies these techniques on the network $N$ and the hyper-network $N^\#_{1,2,3}$ shown in~\Cref{fig::hyper_bab}. 
We have $\hat{z}_{1,1} = -x_1 + x_2$ and $\hat{z}^\#_{1,1} = -x_1 + [1,1.1]\cdot x_2+[0,0.1]$. We can write  $\hat{z}^\#_{1,1}$ as a function of $\hat{z}_{1,1}$ and obtain: $ \hat{z}^\#_{1,1} = z_{1,1} + [0,0.1]\cdot x_2+[0,0.1]$. 
These values pass through ReLU. Since the input $x_2$ ranges over $[0,1]$, we get that $ z^\#_{1,1}$ is between $z_{1,1}$ and ${z}_{1,1}+0.2$.
Thus, the matching dependency of $z^\#_{1,1}$ is:
$ z_{1,1} \leq z^\#_{1,1} \leq z_{1,1} + 0.2$. 
Similarly, the matching dependency of $z^\#_{1,2}$ is: $ z_{1,2} - 0.02 \leq z^\#_{1,2} \leq z_{1,2} $. 
If we precisely encoded all neurons using a boolean variable for each non-input neuron, the computed \propa would be $\beta^*_{1,2,3}=0.21$. 
If we employed linear relaxation for $z^\#_{1,1}$ and $z^\#_{1,2}$, it would eliminate two boolean variables (thereby reducing the problem's complexity) but would result in an overapproximating bound $0.24$.
To balance, relax-if-similar employs linear relaxation only to similar neurons.
In this example, $z^\#_{1,2}$ is similar to $z_{1,2}$ (since $\Delta_{1,2}^l=-0.02$ and $\Delta_{1,2}^u=0$) while $z^\#_{1,1}$ is not similar to $z_{1,1}$  ($\Delta_{1,1}^l=0$ and $\Delta_{1,1}^u=0.2$). Namely, relax-if-similar eliminates one boolean variable (of $z^\#_{1,2}$) \emph{and} enables to compute the precise bound of $\beta^*_{1,2,3}=0.21$. 

\begin{figure}[t]
    \centering
  \includegraphics[width=1\linewidth, trim=0 348 0 0, clip,page=7]{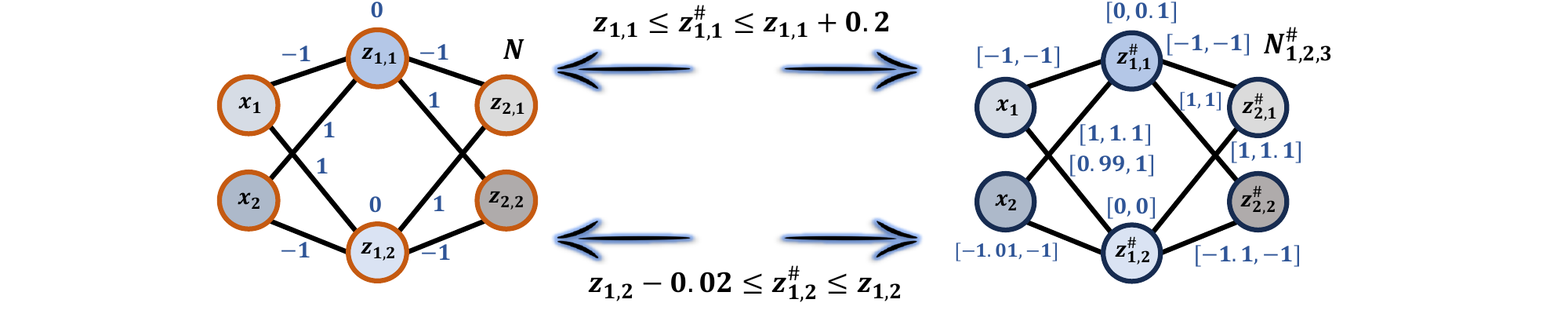}
    \caption{An example of the matching dependencies and relax-if-similar.}
    \label{fig::closness_mutual_encoding}
\end{figure}

\section{Our System: \tool} 
\label{sec:our_approach}
In this section, we describe our system \tool, 
consisting of: (1) \boundtool, computing the \propa of every class using the verification techniques described in~\Cref{sec:overview}, 
and (2)~\reptool, creating \propi label-only access to the network, given the \propa of every class.

\subsection{A System for Computing the \propa}\label{sec:ourapp_sys}
In this section, we describe \boundtool, our system for computing the \propa. 
 \Cref{alg:bab} provides its pseudo-code and \Cref{fig::system_description} illustrates it. 
Its inputs are a training set $D$, a training algorithm $\mathcal{T}$, a network architecture $\widetilde{N}$, and a class $c$. 
It returns the \propa of $c$ (\Cref{dpdbc}). 

\begin{algorithm}[t]
\caption{\boundtool($D$, $\mathcal{T}$, $\widetilde{N}$, $c$)}
\label{alg:bab}
\DontPrintSemicolon
\KwIn{A dataset $D$, a training algorithm $\mathcal{T}$, a network architecture $\widetilde{N}$, and a class $c$.}
\KwOut{The \propa bound $\beta_c$.}
$N$ = $\mathcal{T}(D ,\widetilde{N})$\tcp*{The network given the full dataset}\label{algbab_ln:initializationb}
$N_{D}$ = $\{(x_D,\mathcal{T}(D \setminus\{x_D\},\widetilde{N}))\mid x_D \in D\}$\tcp*{The networks given adjacent datasets}
\label{algbab_ln:initialization}
$\beta_{c}$ = $\infty$ \tcp*{The \propa anytime bound} \label{ln:binit}
$Q$ = $[]$ \tcp*{A priority queue}\label{ln:qinit}
$L$ = $\{D\}$ \tcp*{A list of sets to compute their {\propa}}\label{ln:linit}
\While{True}{
    \For {$S\in L$}{
        $N^\#_S$ = defineHyperNetwork($\{(x_D,N_{-x_D})\in N_{D} \mid x_D \in S\}$)\;\label{ln:hyp}
        $e$ = encodeMILP($N$, $N^\#_S$, $c$)\;\label{ln:milp}
        $I$ = computeDifferenceIntervals($N$, $N^\#_S$)\;\label{ln:diff}
        $e$ = addMatchingDependencies($e$, $I$)\;\label{ln:mat}
        $e$ = relaxIfSimilar($e$,{ $\{[\Delta^l_{m,k},\Delta^u_{m,k}]\in I\mid \Delta^u_{m,k}-\Delta^l_{m,k}\leq \tau\}$})\;\label{ln:rel}
        $\beta_{c,S}$ = MILPSolver($e$)\;\label{algbab_ln:milp_optimization}
        push($Q$, ($S$, $\beta_{c,S}$))\;\label{algbab_ln:push_subgroup}
    }
    ($S$, $\beta_{c,S}$) = pop($Q$)\;\label{algbab_ln:pop_subgroup}
    $\beta_{c}=\beta_{c,S}$\;\label{ln:upb}
     \lIf{$|S|==1$}{break}\label{algbab_ln:break}
    $L$ = partition($k$-\text{means\_elbow}, $\{(x_D,N_{-x_D})\in N_{D} \mid x_D  \in S\}$)\; \label{algbab_ln:kmeans}
   
  }
 \Return{$\beta_{c}$}\label{ln:ret}
\end{algorithm}

\boundtool first trains all networks for every possibility to omit up to one data point (\Cref{algbab_ln:initializationb}--\Cref{algbab_ln:initialization}). 
Then, it performs initializations. First, it initializes the anytime overapproximating bound of the \propa to $\infty$ (\Cref{ln:binit}). 
Next, it initializes a priority queue $Q$ consisting of the sets to branch or bound to an empty queue (\Cref{ln:qinit}). 
Lastly, it initializes a list $L$ consisting of the sets whose \propa is next computed to a list containing the dataset $D$ (\Cref{ln:linit}).
It then runs our BaB (described in~\Cref{sec:overview_bab}). 
Our BaB begins by running the verification analysis for all sets in $L$. 
For every set $S$ in $L$, the analysis constructs the hyper-network $N^\#_{S}$ over the networks in $S$ (\Cref{ln:hyp}). 
Then, it encodes \Cref{dphyp} as a MILP (\Cref{ln:milp}), which we define in~\Cref{sec:our_approach_milp}. 
It then computes the difference intervals (\Cref{ln:diff}), described in~\Cref{sec:overview_match} and defined in~\Cref{sec:our_approach_milp}. 
Accordingly, it adds the matching dependencies (\Cref{ln:mat}) and relaxes neurons whose difference interval is small (\Cref{ln:rel}). 
It then submits the MILP to a MILP solver (\Cref{algbab_ln:milp_optimization}), which returns a bound $\beta_{c,S}$. 
The set $S$ and its bound $\beta_{c,S}$ are pushed to the priority queue $Q$ (\Cref{algbab_ln:push_subgroup}). 
We next describe the role of $Q$. 

As described, after computing $\beta_{c,S}$, the problem of $S$ (\Cref{dphyp}) is either branched or bounded.  
$S$ is bounded if our BaB computes $\beta_{c,S'}$ where $\beta_{c,S}\leq \beta_{c,S'}$ and $|S'|=1$. 
To reduce the number of branched problems, 
 \boundtool \emph{lazily} branches. In each iteration, it branches the problem $S$ that must be branched, which is the one with the maximal bound $\beta_{c,S}$.  
To identify this problem, \boundtool prioritizes in $Q$ the pairs by their bound, such that the next pair that is popped is the one with the maximal bound. 
If \boundtool pops from $Q$ a pair $(S,\beta_{c,S})$ whose $S$ is a singleton $\{x_D\}$, it bounds this problem along with the problems of \emph{all} pairs in the queue (\Cref{algbab_ln:break}). At this point, all problems are handled and \boundtool returns $\beta_{c,S}$ (\Cref{ln:ret}). 

If \boundtool pops from $Q$ a pair $(S,\beta_{c,S})$ such that $|S|>1$, 
it partitions $S$ into several sets and stores them in $L$ to analyze them in the next iteration (\Cref{algbab_ln:kmeans}). 
\boundtool partitions based on the closeness of the networks' weights and biases.
Such partitioning is effective for two reasons. 
First, the closer the parameters are, the closer the networks' functions and their \propa. Consequently, the hyper-network's \propa is expected to be tighter (i.e., smaller), which increases the chances that it will be bounded and not branched.
Second, it reduces the overapproximation error that stems from the abstraction of a hyper-network. 
To partition, \tool employs the $k$-means clustering~\cite{ref82} using the $L_2$ norm computed over all the networks' parameters. 
It dynamically determines the optimal number of clusters $k$ using the elbow method. 

\boundtool is an anytime algorithm: at any point, it can return an overapproximation of the \propa. Initially, the anytime bound is $\infty$ (\Cref{ln:binit}) and after a pair is popped from the priority queue, the anytime bound is the bound of the last popped pair (\Cref{ln:upb}). 

 \begin{figure*}[t]
    \centering
  \includegraphics[width=1\linewidth, trim=0 152 0 0, clip,page=3]{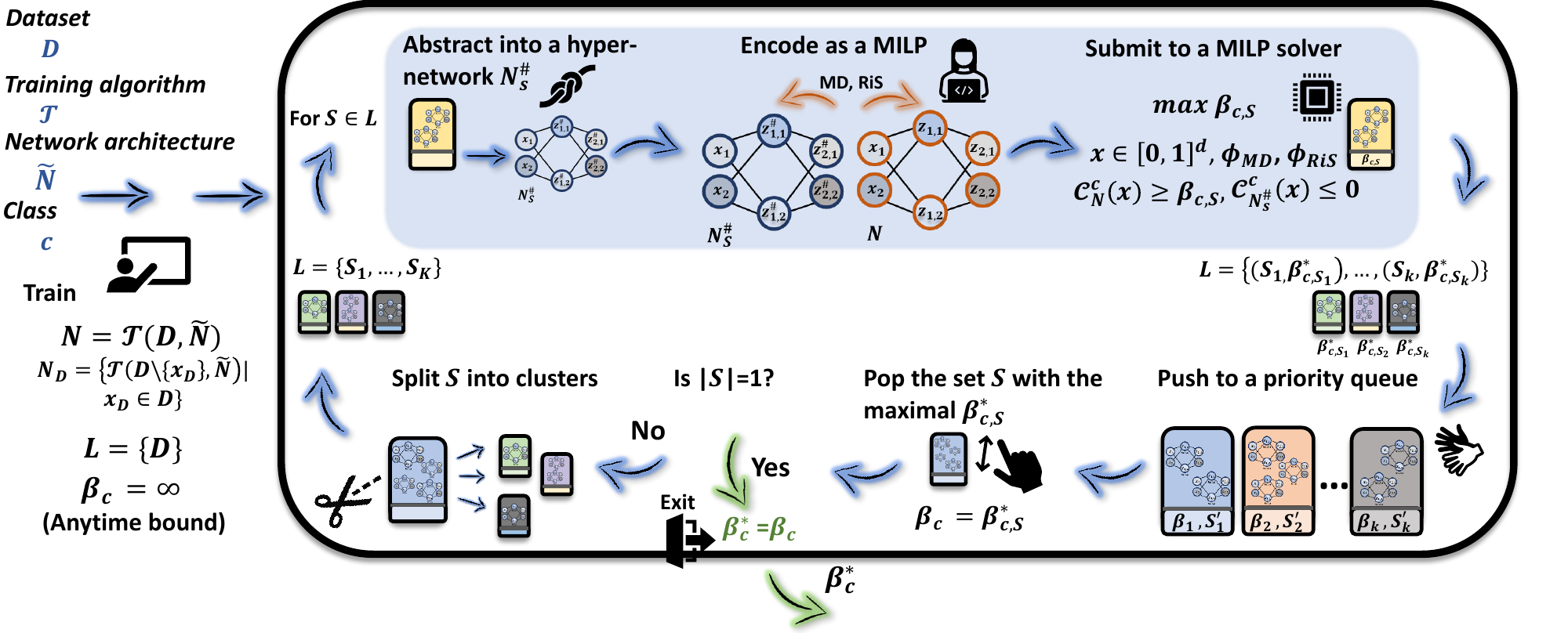}
    \caption{ \tool's component for computing the \propa.}
    \label{fig::system_description}
\end{figure*}
\paragraph{Example}
Consider the example of computing the \propa of $c=0$ for the classifier $N$ shown in~\Cref{fig::hyper_bab} and $N_D=\{N_{-1},N_{-2},N_{-3},N_{-4}\}$. Initially, $\beta_0=\infty$, $Q=[]$, and $L=\{[4]\}$.
In the first iteration of~\Cref{alg:bab}, the hyper-network of $S=[4]$ is constructed (\Cref{fig::hyper_bab}, bottom center) 
and its bound $\beta^*_{1,2,3,4}=0.55$ is computed. Then, the pair $([4],0.55)$ is pushed into the queue: $Q=[([4],0.55)]$. 
Next, this pair is popped from $Q$ and $\beta_0$ is reduced to $0.55$. Since $[4]$ is not a singleton, it is partitioned by clustering. This results in $L=\{\{1,2,3\},\{4\}\}$. Then, another iteration of~\Cref{alg:bab} begins. This iteration first constructs the hyper-network for $S=\{1,2,3\}$ (\Cref{fig::hyper_bab}, bottom left), computes its bound $\beta^*_{1,2,3}=0.21$, and updates the queue: $Q=[(\{1,2,3\},0.21)]$. Then, this iteration constructs the hyper-network for $S=\{4\}$ (\Cref{fig::hyper_bab}, bottom right), computes its bound $\beta^*_{4}=0.52$, and updates the queue: $Q=[(\{4\},0.52),(\{1,2,3\},0.21)]$.
Note that since $0.52>0.21$, the new pair is pushed to the top of the queue. Next, the pair $(\{4\},0.52)$ is popped from $Q$
and $\beta_0$ is reduced to $0.52$. Since $\{4\}$ is a singleton, it bounds \emph{all} pairs in $Q$ and \Cref{alg:bab} returns $\beta_0=0.52$. 

\subsection{The MILP Encoding} 
\label{sec:our_approach_milp} 
In this section, we describe our MILP encoding for computing the \propa of a network and a hyper-network (\Cref{dphyp}). 
We adapt the MILP encoding from~\citet{ref_42}, originally designed to analyze a network's local robustness, to our property defined over any input and over a network and a hyper-network. 
\
We begin with background and then present our adaptations. 

\paragraph{Background}
The MILP verifier of~\citet{ref_42} takes as inputs a classifier $N$, an input $x$ classified as $c$, and a neighborhood $\mathcal{I}(x)$, defined by intervals for each input entry (i.e., $x' \in \mathcal{I}(x)$ if and only if $\forall k.\ x'_k \in [l_k, u_k]$). 
The verifier determines whether $N$ is locally robust at $\mathcal{I}(x)$, i.e., whether it classifies all inputs in $\mathcal{I}(x)$ as $c$. 
For the input neurons, 
the verifier assigns each a real-valued variable $z_{0,k}$, constrained by its interval: $\forall k.\ l_k \leq z_{0,k} \leq  u_k$. 
For the non-input neurons, it assigns two variables: $\hat{z}_{m,k}$ for the weighted sum (as defined in~\Cref{sec:preliminary}) 
and ${z}_{m,k}$ for the ReLU application. It also introduces 
concrete bounds $l_{m,k}, u_{m,k}\in \mathbb{R}$ for $\hat{z}_{m,k}$.
The verifier encodes $\hat{z}_{m,k}$ as is using one equality constraint, and it encodes ReLU using four constraints, defined over $\hat{z}_{m,k}$ and over a boolean variable $a_{m,k}$ indicating whether the ReLU is in its inactive state (i.e., outputs zero) or active state (i.e., outputs $\hat{z}_{m,k}$).
To determine local robustness, the verifier adds the constraint 
$z_{L,c} - \max_{c'\neq c } z_{L,c'} \leq 0$. If this MILP has no solution, $N$ is locally robust at $\mathcal{I}(x)$; otherwise, it is not. 

\paragraph{Our MILP encoding} 
We next present our MILP encoding for 
the \propa of a set $S$ (\Cref{main_problem}), which relies on several adaptations. 
First, since \propa is defined over a network and a hyper-network, it has two copies of the variables: for $N$ and for the hyper-network $N^\#_S$. For $N$, it has the same constraints for the non-input neurons (Eq. (\ref{eq3:row4}) and (\ref{eq3:row6})).
For $N^\#_S$, it has the same constraints for the ReLU computations (Eq. (\ref{eq3:row7})). 
For the weighted sums, since the hyper-network's parameters are intervals and not real numbers,
it replaces the equality constraint by two inequality constraints (Eq. (\ref{eq3:row5})). The correctness of these constraints follows from interval arithmetic and since the input entries and the ReLU's outputs are non-negative.
Second, to encode that the network and hyper-network accept the same input, it introduces a variable $x$ and defines $z_{0,k}=x_k$ and 
$z_{0,k}^\#=x_k$, for all $k$ (Eq. (\ref{eq3:row2})).
Third, since \propa is a global property, there is no neighborhood and each input entry $x_k$ is bounded in the domain $[0,1]$ (Eq. (\ref{eq3:row2})). 
 %
Fourth, it introduces a bound $\beta_{c,S}$ and adds the objective $\max \beta_{c,S}$ (Eq. (\ref{eq3:row1})) as well as the constraints $z_{L,c} -  z_{L,c'} \geq \beta_{c,S}$ for all $c'\neq c$, 
for encoding $\mathcal{C}_{N}^c(x) \geq \beta_{c,S}$,
and 
$z^\#_{L,c} - \max_{c'\neq c} z^\#_{L,c'} \leq 0$, for encoding $\mathcal{C}_{{N}_{S}^\#}^{c}(x) \leq 0$ (Eq. (\ref{eq3:row3})). 
It expresses $\max_{c'\neq c} z^\#_{L,c'}$ by the Big M method~\cite{ref_73}, which relies on a large constant $M$ and boolean variables $a_{c'}$ for every $c'\neq c$. 
Fifth, it adds the matching dependencies $\phi_{MD}$ and linear relaxation constraints $\phi_{RiS}$ (Eq. (\ref{eq3:row2})), for neurons in the hyper-network whose difference interval is small (described shortly). 
Overall, given a classifier $N$, a hyper-network $N^\#_S$ with parameters $w_{m,k,k'}^\#=[\underline{w^\#}_{m,k,k'},\overline{w^\#}_{m,k,k'}]$ and $b_{m,k}^\#=[\underline{b^\#}_{m,k},\overline{b^\#}_{m,k}]$, and a class $c$, our encoding is given in~\Cref{main_problem} below:
%

\begin{subequations}\label{main_problem}
\begin{equation}\label{eq3:row1}
  \begin{gathered}
     \max \beta_{c,S} \hspace{0.5cm}\text{subject to}
  \end{gathered}
\end{equation}
\begin{equation}\label{eq3:row2}
      \phi_{MD}; \hspace{0.27cm} \phi_{RiS}; \hspace{0.27cm} x\in[0,1]^d;\hspace{0.27cm}
      \forall k:\hspace{0.2cm}\;z_{0,k}=x_{k};\hspace{0.3cm}z^\#_{0,k}=x_{k}
\end{equation}
\begin{equation}\label{eq3:row3}
      \forall {c'\neq c}.\ z_{L,c}-z_{L,c'}\geq\beta_{c,S};\hspace{0.27cm}\;\;\;
      \forall {c' \neq c}.\  z^\#_{L,c}-z^\#_{L,c'} \leq M\cdot(1-\alpha_{c'});\;\;\sum_{c' \neq c} \alpha_{c'} \geq 1
\end{equation}
\begin{equation}\label{eq3:row4}
     \forall m>0, \forall k:\hspace{0.5cm}\;\hat{z}_{m,k}=b_{m,k}+\sum_{k'=1}^{k_{m-1}}{w}_{m,k,k'}\cdot{z}_{m-1,k'} \hspace{0.5cm}
\end{equation}
\begin{equation}\label{eq3:row5}
     \hat{z}^\#_{m,k}\geq \underline{b^\#}_{m,k} +\sum_{k'=1}^{k_{m-1}}\underline{w^\#}_{m,k,k'}\cdot{z^\#}_{m-1,k'};\hspace{0.5cm}
     \hat{z}^\#_{m,k}\leq \overline{b^\#}_{m,k}+\sum_{k'=1}^{k_{m-1}}\overline{w^\#}_{m,k,k'}\cdot{z^\#}_{m-1,k'}
\end{equation}
\begin{equation}\label{eq3:row6}
      {z}_{m,k}\geq0; \hspace{0.5cm} {z}_{m,k}\geq \hat{z}_{m,k}; \hspace{0.5cm} {z}_{m,k} \leq u_{m,k}\cdot a_{m,k}; \hspace{0.5cm}
      {z}_{m,k} \leq \hat{z}_{m,k}-l_{m,k}(1-a_{m,k})
\end{equation}
\begin{equation}\label{eq3:row7}
     {z}^\#_{m,k}\geq0;\hspace{0.5cm} {z}^\#_{m,k}\geq \hat{z}^\#_{m,k};\hspace{0.5cm}
    {z}^\#_{m,k} \leq u^\#_{m,k}\cdot a^\#_{m,k};\hspace{0.5cm}
    {z}^\#_{m,k} \leq \hat{z}^\#_{m,k}-l^\#_{m,k}(1-a^\#_{m,k})
\end{equation}
\end{subequations}

\paragraph{Difference intervals} The matching dependencies and relax-if-similar techniques rely on the difference intervals.  
A difference interval $I_{z_{m,k}} = [\Delta^l_{z_{m,k}}, \Delta^u_{z_{m,k}}] \in \mathbb{R}^2$ is defined for every matching neurons and is computed by bound propagation and interval arithmetic.
To formally define it, 
we define for each weight $w_{m,k,k'}$ an interval capturing the difference: $I_{w_{m,k,k'}}=w^\#_{m,k,k'}-[w_{m,k,k'},w_{m,k,k'}]$, where $w_{m,k,k'}$ is the real-valued weight in $N$ and $w^\#_{m,k,k'}$ is the weight interval in $N^\#_S$. Similarly, $I_{b_{m,k}}=b^\#_{m,k}-[b_{m,k},b_{m,k}]$. 
We define the difference intervals inductively on $m$.
For the input layer, $\forall k\in[d].\ I_{z_{0,k}}=[0,0]$. 
For the other layers, for all $m>0$ and $k\in[k_m]$: 
\begin{align*}
I_{\hat{z}_{m,k}}=  
\left( {b_{m,k}^\#}+\sum_{k'=1}^{k_{m-1}}{w^\#_{m,k,k'}}\cdot  z^\#_{m-1,k'}\right) - \hat{z}_{m,k} =\phantom{aaaaaaa}\\
b_{m,k}+I_{b_{m,k}}+\sum_{k'=1}^{k_{m-1}}(w_{m,k,k'}+I_{w_{m,k,k'}})\cdot  (z_{m-1,k'}+I_{z_{m-1,k'}}) - \hat{z}_{m,k}\\
=I_{b_{m,k}}+\sum_{k'=1}^{k_{m-1}} w_{m,k,k'} \cdot I_{z_{m-1,k'}} + I_{w_{m,k,k'}}\cdot  (z_{m-1,k'}+I_{z_{m-1,k'}}) 
\end{align*}
Here, $+$ and $\cdot$ are the standard addition and multiplication of two intervals (when writing a variable $z$, we mean the interval $[z,z]$).  
     For the ReLU computation, we bound the difference as follows:
%
\begin{align*}
\text{ReLU}(\hat{z}^\#_{m,k}) - \text{ReLU}(\hat{z}_{m,k})=
\text{ReLU}(\hat{z}_{m,k} + I_{\hat{z}_{m,k}})- \text{ReLU}(\hat{z}_{m,k})\leq \\
\text{ReLU}(\hat{z}_{m,k}) + \text{ReLU}(\Delta^u_{\hat{z}_{m,k}})- \text{ReLU}(\hat{z}_{m,k})=\max(0,\Delta^u_{\hat{z}_{m,k}})
\end{align*}
Similarly, 
\begin{align*}
\text{ReLU}(\hat{z}^\#_{m,k}) - \text{ReLU}(\hat{z}_{m,k})=
\text{ReLU}(\hat{z}_{m,k} + I_{\hat{z}_{m,k}})- \text{ReLU}(\hat{z}_{m,k})\geq 
 \text{ReLU}(\hat{z}_{m,k} - (-\Delta^l_{\hat{z}_{m,k}}) )- \\ \text{ReLU}(\hat{z}_{m,k}) \geq
 \text{ReLU}(\hat{z}_{m,k}) - \text{ReLU}(-\Delta^l_{\hat{z}_{m,k}})- \text{ReLU}(\hat{z}_{m,k})=
  -\max(0,-\Delta^l_{\hat{z}_{m,k}})
\end{align*}
The inequality transitions follow from the triangle inequalities for ReLU.
Overall,
$I_{z_{m,k}}=[-\max(0,-\Delta^l_{\hat{z}_{m,k}}),\max(0,\Delta^u_{\hat{z}_{m,k}})]$. 
It is possible to compute a tighter difference interval~\cite{ref_1,ref_2}, but with a longer computation, and in practice it is not required in our setting. 

\paragraph{Matching dependencies}
 The matching dependencies are linear constraints encoding the difference intervals to prune the MILP's search space. 
Given the intervals, the combined constraint is: $$\phi_{MD}=\bigwedge_{m\in[L],k\in [k_m]} \left ( z^\#_{m,k} \geq z_{m,k}+\Delta^l_{z_{m,k}}\wedge  z^\#_{m,k} \leq z_{m,k}+\Delta^u_{z_{m,k}}\right )$$

\paragraph{Relax-if-similar} To reduce the MILP's complexity, relax-if-similar overapproximates neurons in the hyper-network whose difference interval is small (its size is below a small threshold $\tau$). The overapproximation replaces their ReLU constraints defined over a boolean variable by linear constraints overapproximating the ReLU computation as proposed by~\citet{Ehlers17}.
Namely, it replaces the constraints of $z^\#_{m,k}$ for these neurons (Eq. (\ref{eq3:row7})) by:
\begin{align*}
\phi_{RiS}=\bigwedge_{\Delta^u_{z_{m,k}}-\Delta^l_{z_{m,k}}\leq \tau }
 \left(z^\#_{m,k} \geq 0 \wedge z^\#_{m,k} \geq \hat{z}^\#_{m,k}\wedge z^\#_{m,k} \leq \frac{u^\#_{m,k}}{u^\#_{m,k} - l^\#_{m,k}} (\hat{z}^\#_{m,k} - l^\#_{m,k}) \right) 
\end{align*}

\begin{algorithm}[t]
\caption{\reptool($D$, $\mathcal{T}$, $\widetilde{N}$, $\varepsilon$)}
\label{alg:repair}
\SetKwProg{Fn}{Function}
\DontPrintSemicolon
\Fn{}{
\DontPrintSemicolon
$N$ = $\mathcal{T}(D,\widetilde{N})$\;\label{ln:reptrain}
$\beta\-Set$ =$\{\boundtool(D,\mathcal{T},\widetilde{N},c) \mid c\in C\}$\;\label{ln:repbounds}
$M$ = $[]$\tcp*{A dictionary of inputs for which the exponential mechanism ran} \label{ln:repM}
}
\Fn{Access($x$)}{
\DontPrintSemicolon
\KwIn{The input to the network $x$.}
\KwOut{A label, such that the label-only query satisfies $\varepsilon$-iDP.}
\lIf{$M[x] \neq \bot$}{\Return{$M[x]$}}\label{ln:repif}
$scores$ = $N(x)$\;\label{ln:repscore}
 $c$ = $\text{argmax}_{c\in C} scores$\;  \label{ln:repclass}
 \lIf{$scores[c]-max_{c'\neq c}scores[c']> \beta\-Set[c]$}{\Return{c}}\label{ln:repscorbound}
 $M[x]$ = ExponentialMechanism($u_{x,\mathcal{T},\widetilde{N}},D,c,\varepsilon)$\;\label{ln:repexp}
 \Return{$M[x]$}\label{ln:repret}
  }
 \end{algorithm}
\subsection{Our \propi Label-Only Access to the Network Classifier} 
In this section, we describe \reptool, our system for creating label-only access to a neural network classifier that is individually differentially private (\propi).
\reptool (\Cref{alg:repair}) takes as inputs a dataset $D$, a training algorithm $\mathcal{T}$, a network architecture $\widetilde{N}$, and a privacy budget~$\varepsilon$.
It first trains the network $N$ and runs \boundtool for every class $c$ to obtain all \propa (\Cref{ln:reptrain}--\Cref{ln:repbounds}). 
It then initializes a dictionary $M$ of inputs and their labels, for inputs whose labels are selected by the exponential mechanism (\Cref{ln:repM}). This is used in case the attacker repeatedly submits the same input. 
Then, given an input $x$, the \propi access is defined as follows.  
If $M$ contains $x$, it returns the label in $M[x]$ (\Cref{ln:repif}).
Otherwise, it passes $x$ through $N$ to compute the output vector~$N(x)$ (\Cref{ln:repscore}). It then identifies the predicted label $c$, which is the label with the maximal score (\Cref{ln:repclass}). 
It checks whether its classification confidence is greater than $\beta_{c}$, and if so, it returns~$c$ (\Cref{ln:repscorbound}). 
Otherwise, it runs the exponential mechanism as described in~\Cref{sec:naive}. 
 It stores the resulting label in $M[x]$ and returns it (\Cref{ln:repexp}--\Cref{ln:repret}). 
We assume \reptool's memory is as large as the attacker's memory. Namely, after $M$ reaches the memory limit, \reptool overrides the oldest entry, which the attacker also forgets.

\begin{restatable}[]{theorem}{ftc}
Given a dataset $D$, a training algorithm $\mathcal{T}$, a network architecture $\widetilde{N}$, and a privacy budget $\varepsilon$,
\reptool provides an $\varepsilon$-\propi version of every label-only query $f_{x,\mathcal{T},\widetilde{N}}(D)$, for $x \in [0,1]^d$, without losing accuracy for queries whose $x$ has a confidence larger than the \propa.
\end{restatable} 

 A proof sketch is provided in 
 \ifthenelse{\EXTENDEDVER<0}{\Cref{sec:proofs}.
}{\citet[Appendix A]{ref_105}.}

\section{Evaluation}
\label{sec:eval}

In this section, we evaluate \tool's effectiveness in providing \propi label-only access to a neural network classifier. 
We run all the experiments on a dual AMD EPYC 7713 64-Core Processor@2GHz server with 2TB RAM and eight A100 GPUs running an Ubuntu 20.04.1 OS. 
We begin with the experiment setup: implementation, datasets, and networks. 
We then compare \tool to baselines. Finally, we provide an ablation study, showing the effectiveness of \tool's components. 
\paragraph{Implementation} 
We implemented \tool in Julia 1.8.3 as a module extending MIPVerify~\cite{ref_42}.
To solve MILPs, \tool uses Gurobi 11.0.1~\cite{gurobi}.  
We run \tool with a timeout of eight hours. The timeout of a single MILP (for solving~\Cref{main_problem} for a given set~$S$) is 40 minutes. \tool is parallelizable: it uses 32 workers that pop from the priority queue $Q$ sets $S$ to branch or bound. For the relax-if-similar technique, it uses $\tau=0.01$. 

\paragraph{Networks}
We evaluate \tool over four datasets (described in  
 \ifthenelse{\EXTENDEDVER<0}{\Cref{sec:appeval}}{\citet[Appendix B]{ref_105}}):
Cryptojacking~\cite{ref_43} (Crypto), Twitter Spam Accounts~\cite{ref_44} (Twitter), Adult Census~\cite{ref_45} (Adult), and Default of Credit Card Clients~\cite{ref_47} (Credit).
We consider three fully-connected neural networks with architectures of 2$\times$50, 2$\times$100, and 4$\times$30, where the first number is the number of intermediate layers and the second number is the number of neurons in each of these layers. 
Additionally, we evaluate \tool on a convolutional neural network (CNN) architecture that has two convolutional layers followed by a fully-connected layer. 
The activation function in all networks is ReLU. 
We train the networks using SGD, over 50 epochs, the batch size is 1024 (except for the Crypto dataset, where the batch size is 100), and the initial learning rate is $\eta = 0.1$. 
Although the network sizes may seem small compared to the common network sizes evaluated by local robustness verifiers (which are typically evaluated over image classifiers that have larger input dimensionality), our network sizes are consistent with those used in previous work evaluating the datasets we consider~\cite{ref_8,ref_54,ref_55,ref_56}. We remind that \tool analyzes a much more challenging property than local robustness or fairness: \propa is a global property, over any input, and over a very large number of classifiers (the size of the dataset plus one).

\ifthenelse{\WITHNAIVE>0}{

\begin{table}[t]
\small
\begin{center}
\caption{The accuracy of \tool ($Acc_{w/o-p}$ is the model's accuracy without privacy protection).}
\begin{tabular}{lll cccccccc}
\toprule
Dataset & Model & $Acc_{w/o-p}$ & \multicolumn{8}{c}{\tool}  \\
\cmidrule(lr){4-11} 
        &       &         & $\beta_0$ & $T_0$ &$\beta_1$& $T_1$ & $Acc$ & $Acc$ & $Acc$         & $AT$\\  
        &       &         &            &[h] &   &      $[h]$  & $[\%]$    & $[\%]$     & $[\%]$     & $[ms]$               \\  
        &       &         &            & &   &       &  $\varepsilon=0$       & $\varepsilon=0.2$ & $\varepsilon$=1&     \\  

\midrule
Crypto   & 2$\times$50    &  99.7   &    1.78&  0.1   & 2.95  & 0.07           & 97.3    &   97.3 &   98.0          &   0.1 \\ 
Crypto   & 2$\times$100    &  99.7   &    2.02& 1.1   & 2.86 & 1.0            & 97.1    &   97.5 &   97.8           &    0.2\\   
Crypto   & 4$\times$30    &  99.8   &    4.23 & 0.1 & 6.12 & 0.1              & 94.7    &   95.0 &   95.5           &    0.2\\     
\midrule
Twitter   & 2$\times$50   &  89.7   &    0.51 & 1.6 & 0.54  & 3.7             & 89.3    &   89.3 &   89.4           &     0.1\\  
Twitter   & 2$\times$100   &  89.6   &   0.86 & 4.9 & 0.96  & 5.5            & 88.1    &   88.2 &   88.5            &     0.2\\ 
\midrule
Adult   & 2$\times$50   &  83.1   &    0.55  & 2.9& 0.63  & 3.0               & 82.4    &   82.4 &   82.5           &     0.1\\     

Adult   & 2$\times$100  &  83.3   &    0.93 &  4.3& 0.91 & 4.1               & 80.7    &   80.8 &   81.0            &     0.1\\ 

Adult   & $conv$ &  79.2   &    0.29    & 0.2& 0.52 & 0.1                  & 80.1    &   80.1 &   80.0            &   0.2\\ 
\midrule
Credit   & 2$\times$50    &  81.9   &    0.33 &  5.8  & 0.34  & 6.5         & 81.7    &   81.8 &   81.8             &    0.1\\ 
Credit   & 2$\times$100   &  81.8   &    0.42  & 4.4  & 0.44  & 3.1          & 81.4    &   81.4 &   81.5            &  0.2\\ 
Credit   & $conv$  &  80.8   &    0.75&  0.3  & 0.57   & 0.2                & 79.9    &   80.1 &   80.3             &     0.3\\     
\bottomrule        
\end{tabular}
    \label{tab:results}
\quad
\end{center}
\end{table}

\paragraph{Baselines} 
We compare \tool to four baselines. First, the naive algorithms described in~\Cref{sec:naive}: 
Naive-Noise, invoking the exponential mechanism for every input, and
Naive-\propi, which passes every input through $|D|+1$ networks and invokes the exponential mechanism if there is a disagreement between them. Second,   
we compare to two DP training algorithms: 
DP-SGD~\cite{ref_22}, using its PyTorch implementation\footnote{\url{https://github.com/ChrisWaites/pyvacy.git}}, and ALIBI~\citep{ref_58}, using the authors' code\footnote{\url{https://github.com/facebookresearch/label\_dp\_antipodes.git}}. 
DP-SGD provides a DP guarantee to all the network's parameters and ALIBI provides a DP guarantee only for the classifier's output. 
Both of them inject Gaussian noise into the network's training process to obtain a user-specified $\varepsilon$-DP guarantee. 
In our experiments, 
we employ a binary search to determine the appropriate noise level for achieving the desired $\varepsilon$.
The binary search increases the computation time by a few milliseconds and thus poses a minor increase to their overall computation time, which is multiple seconds. 
These two baselines provide a stronger privacy guarantee (DP) than \propi, but any $\varepsilon$-DP algorithm is also $\varepsilon$-\propi~\cite{ref_88}. We compare to them because they represent the current approach for providing privacy guarantees to neural networks. 
We note that we do not compare to other works guaranteeing other variants of DP~\cite{ref_94,ref_95,ref_97,ref_98,ref_99,ref_8,ref_100,ref_101}, since their guarantee is incomparable to \propi.

\begin{table}[t]
\small
\begin{center}
\caption{The accuracy of Naive-Noise and Naive-\propi ($Acc_{w/o-p}$ is the accuracy without privacy protection).}
\begin{tabular}{lll cccc cccc}
\toprule
Dataset & Model & $Acc_{w/o-p}$  & \multicolumn{4}{c}{Naive-Noise} & \multicolumn{4}{c}{Naive-\propi} \\
\cmidrule(lr){4-7} \cmidrule(lr){8-11}  
        &       &          &  $Acc$           &  $Acc$              & $Acc$            & $AT$              &  $Acc$            &  $Acc$           & $Acc$          & $AT$   \\  
        &       &          &  $[\%]$          & $[\%]$              & $[\%]$           & $[ms]$            &  $[\%]$           & $[\%]$           & $[\%]$                & $[s]$\\  
        &       &          &  $\varepsilon=0$ & $\varepsilon=0.2$   &  $\varepsilon=1$ &                   &$\varepsilon=0$    & $\varepsilon=0.2$ &  $\varepsilon=1$                   & \\  

\midrule
Crypto   & 2$\times$50    &  99.7   &   49.9   & 52.3   & 62.1     &  0.1 &   99.4    &  99.4   & 99.5     &  3          \\ 
Crypto   & 2$\times$100   &  99.7   &   50.1   &  52.7   & 62.1       & 0.2 &   99.3   &  99.2   & 99.3       & 4      \\   
Crypto   & 4$\times$30    &  99.8   &    50.0 &   52.4 &   62.0       & 0.2 &    99.7 &   99.7 &   99.7       & 4        \\     
\midrule
Twitter   & 2$\times$50   &  89.7   &     49.9   &   51.9  & 59.6    &  0.2 &     89.5   &   89.5  &  89.6    &  43        \\  
Twitter   & 2$\times$100  &  89.6       &   50.1   &  52.2   &  59.9      &   0.3 &   89.3   &  89.4   &  89.4      &   52       \\ 
\midrule
Adult   & 2$\times$50     &  83.1       &   49.9  &  51.7   & 58.2        &    0.1  &   83.1  &  82.9   & 83.1        &    32   \\     

Adult   & 2$\times$100    &  83.3        &   49.9  & 51.7    & 58.9        &   0.2  &   83.2  & 83.1    & 83.3        &   34     \\ 

Adult   & $conv$          &  79.2        &   49.9  &  51.4   & 57.2        &  0.2   &   79.2  &  79.5   & 79.5        &  40    \\ 
\midrule
Credit   & 2$\times$50    &  81.9          &  50.1    &   51.6  & 57.7       & 0.1    &  81.8    &   81.9  & 81.9       & 19    \\ 
Credit   & 2$\times$100   &  81.8         &   50.1   &   51.6  &  57.9      &  0.2  &   81.7   &   81.7  &  81.8      &  21   \\ 
Credit   & $conv$         &  80.8          &    49.9  &   51.6  & 57.5       & 0.3  &    80.9  &   80.7  & 80.7       & 26   \\   
\bottomrule        
\end{tabular}
    \label{tab:results3}
\quad
\end{center}
\end{table} 

\paragraph{\tool's performance}
We begin by evaluating the performance of \tool over fully-connected and convolutional networks for all our datasets, compared to all four baselines.
For \tool, \Cref{tab:results} reports the computed \propa for each label $\beta_0$ and $\beta_1$, their computation time in hours ($T_0$ and $T_1$), 
the test accuracy $Acc$ for three values of $\varepsilon$ ($0$, $0.2$, and $1$), and the access time in milliseconds ($AT$), i.e., the time to return a label for a given input. 
For Naive-Noise and Naive-\propi, \Cref{tab:results3} reports the test accuracy $Acc$ and the access time in milliseconds and seconds, respectively. 
For the DP training algorithms, \Cref{tab:results2} reports the test accuracy $Acc$ and the training time in seconds ($CT$). 
We run the DP training algorithms with the same values of $\varepsilon$, except that we replace the value $0$ with $0.02$, since DP cannot be obtained for $\varepsilon=0$. In fact, $0.02$ is the lowest value that our baselines can provide, and it requires adding a very high noise (the standard deviation of the noise distribution is $10,000$). 
We note that we focus on these values of $\varepsilon$ since commonly a value of $\varepsilon$ is considered to provide a strong DP guarantee when $\varepsilon<1$ and a moderate guarantee when $\varepsilon \in [1,3]$. 
Commonly, DP training algorithms, including our baselines, focus on guarantees for $\varepsilon\geq 1$.
We note that when providing the baselines $\varepsilon$ smaller than $0.02$, it results in a very large noise, rendering the computation practically infeasible (due to computer arithmetic issues). 
The results show that, for $\varepsilon=0$, \tool provides a $0$-\propi guarantee with only 1.4\% accuracy decrease. As expected, for larger values of $\varepsilon$, \tool lowers the accuracy decrease: it provides a $0.2$-\propi guarantee with a 1.3\% accuracy decrease and a $1$-\propi guarantee with a 1.1\% accuracy decrease.
To compute the \propa, \tool runs for 4.8 hours on average (this computation is run once for every network). 
Its \propi-access time is at most one millisecond and on average it is 0.16 milliseconds.
This time is comparable to the access time of standard networks (without \propi guarantees), with a negligible overhead. 
For Naive-Noise, the accuracy decrease is very high because it invokes the exponential mechanism for every input (unlike \tool and Naive-\propi). Its accuracy decrease is 
  38.1\%, 36.1\%, and 28.7\% for $\varepsilon=0$, $\varepsilon=0.2$, and $\varepsilon=1$, respectively. Like \tool, its access time is negligible.  
Naive-\propi obtains the minimal accuracy decrease, since it precisely identifies the set of leaking inputs (unlike \tool that overapproximates them with the \propa). Its accuracy decrease is 
  0.12\%, 0.1\%, and 0.05\% for the privacy budgets $\varepsilon=0$, $\varepsilon=0.2$, and $\varepsilon=1$, respectively. 
  However, it has to store throughout its execution all $|D|+1$ classifiers and its access time is high: 25 seconds on average (since it passes every input through all classifiers). This severely undermines its utility in real-time applications of neural networks, providing an interactive communication. 
The accuracy decrease of DP-SGD, for $\varepsilon\in\{0.02,0.2,1\}$, is on average 42.3\%, 14.2\%, and 14.3\%, respectively (29.8x, 10.7x, and 12.9x higher than \tool). 
The accuracy decrease of ALIBI for the same DP guarantees is on average 28.6\%, 19.3\%, and 2.7\% (20.1x, 14.5x, and 2.4x higher than \tool). 
Even worse, on some networks (Crypto 4$\times$30, Twitter 2$\times$50, and Twitter 2$\times$100), the baselines train networks whose accuracies are 50\% (like the accuracy of a random classifier).  
If we ignore these three networks, DP-SGD's accuracy decrease is 40.13\%, 3.99\%, and 3.9\%, which is still significantly higher than \tool (by 28.3x, 3.0x, and 3.5x), while ALIBI's decrease is 23.4\%, 10.8\%, and 2.7\% (higher by 16.4x, 8.2x, and 2.4x). 
The baselines are faster than \tool's analysis time (for computing the \propa): 
DP-SGD completes within 19.9 seconds and ALIBI within 98.2 seconds. 
Although they are significantly faster, their training algorithms are coupled to a specific privacy budget $\varepsilon$: if a user wishes to update the privacy guarantee, they need to retrain the classifier. In contrast, \tool's \propi-access can easily configure the $\varepsilon$ (the \propa are computed once for a classifier).
Further, as we show, computing the \propa allows us to achieve \propi guarantees with a small accuracy decrease. 
}{
\paragraph{Baselines} 
We compare \tool to two baselines. 
DP-SGD~\cite{ref_22}, using its PyTorch implementation\footnote{https://github.com/ChrisWaites/pyvacy.git}, and ALIBI~\citep{ref_58}, using the authors' code\footnote{https://github.com/facebookresearch/label\_dp\_antipodes.git}. 
DP-SGD provides a DP guarantee to all the network's parameters, whereas ALIBI provides a DP guarantee only for the classifier's predicted labels. 
Both baselines propose a training algorithm that involves injecting Gaussian noise into the network's training process to obtain a user-specified $\varepsilon$-DP guarantee. 
In our experiments, 
we employ a binary search to determine the appropriate noise level for achieving the desired $\varepsilon$.
The binary search increases the computation time by a few milliseconds and thus poses a minor increase to their overall computation time, which is about few seconds. 
We note that although the DP-SGD baseline approaches provide a stronger DP privacy guarantee than \propi (in fact, DP overapproximates \propi as shown in~\cite{ref_88}), we still compare to them because they represent the main line of research providing privacy guarantees in neural networks. Additionally, they allow us to obtain $\varepsilon$-\propi by proving their $\varepsilon$-DP guarantees.

\begin{table}[t]
\small
\begin{center}
\caption{The accuracy of \tool ($Acc_{w/o-p}$ is the model's accuracy without privacy protection).}
\begin{tabular}{lll ccccccc}
\toprule
Dataset & Model & $Acc_{w/o-p}$ & \multicolumn{7}{c}{\tool}  \\
\cmidrule(lr){4-10} 
        &       &         & $\beta_0$ & $T_0$ &$\beta_1$& $T_1$ & $Acc$ & $Acc$ & $Acc$         \\  
        &       &         &            &[h] &   &      $[h]$  & $[\%]$    & $[\%]$     & $[\%]$               \\  
        &       &         &            & &   &       &  $\varepsilon=0$       & $\varepsilon=0.2$ & $\varepsilon$=1     \\  

\midrule
Crypto   & 2$\times$50    &  99.7   &    1.78&  0.1   & 2.95  & 0.07           & 97.3    &   97.3 &   98.0            \\ 
Crypto   & 2$\times$100    &  99.7   &    2.02& 1.1   & 2.86 & 1.0            & 97.1    &   97.5 &   97.8             \\   
Crypto   & 4$\times$30    &  99.8   &    4.23 & 0.1 & 6.12 & 0.1              & 94.7    &   95.0 &   95.5            \\     
\midrule
Twitter   & 2$\times$50   &  89.7   &    0.51 & 1.6 & 0.54  & 3.7             & 89.3    &   89.3 &   89.4                \\  
Twitter   & 2$\times$100   &  89.6   &   0.86 & 4.9 & 0.96  & 5.5            & 88.1    &   88.2 &   88.5            \\ 
\midrule
Adult   & 2$\times$50   &  83.1   &    0.55  & 2.9& 0.63  & 3.0               & 82.4    &   82.4 &   82.5           \\     

Adult   & 2$\times$100  &  83.3   &    0.93 &  4.3& 0.91 & 4.1               & 80.7    &   80.8 &   81.0                 \\ 

Adult   & $conv$ &  79.2   &    0.29    & 0.2& 0.52 & 0.11                  & 80.1    &   80.1 &   80.0               \\ 
\midrule
Credit   & 2$\times$50    &  81.9   &    0.33 &  5.8  & 0.34  & 6.5         & 81.7    &   81.8 &   81.8                 \\ 
Credit   & 2$\times$100   &  81.8   &    0.42  & 4.4  & 0.44  & 3.1          & 81.4    &   81.4 &   81.5             \\ 
Credit   & $conv$  &  80.8   &    0.75&  0.3  & 0.57   & 0.2                & 79.9    &   80.1 &   80.3                  \\     
\bottomrule        
\end{tabular}
    \label{tab:results}
\quad
\end{center}
\end{table} 

\paragraph{\tool's performance}
We begin by evaluating the performance of \tool over fully connected and convolutional networks, using the datasets: Crypto, Twitter, Adult, and Credit. 
For each case, we report the computed \propa for each label $\beta_0$ and $\beta_1$, along with the times $T_0$ and $T_1$ required to compute them in hours ($CT$). 
Additionally, we provide the test accuracy $Acc$ for three values of $\varepsilon$: $\varepsilon=0$, $\varepsilon=0.2$, and $\varepsilon=1$. 
For the DP-SGD based baselines, we report the test accuracy and the training time in seconds. 
We run the DP-SGD based approach with different values of privacy budget $\varepsilon \in \{0.02,0.2,1\}$.
These values demonstrate strong privacy guarantees (for $\varepsilon<1$) and moderate guarantees $(\varepsilon \in [1,3]$). 
Previous  DP-SGD based work, including our baselines, often focus on guarantees where $\varepsilon\geq 1$.
We note that when providing the baselines $\varepsilon$ smaller than $0.02$, it results in a very large noise, rendering the computation practically infeasible (due to computer arithmetic issues). 
\Cref{tab:results} shows the results of \tool, while \Cref{tab:results2} shows the results of the two DP-SGD based baselines. 
The results indicate that \tool provides a $0$-\propi guarantee with only 1.4\% accuracy decrease. As expected, for larger values of $\varepsilon$, \tool lowers the accuracy decrease. \tool provides a $0.2$-\propi guarantee with a 1.3\% accuracy decrease and a $1$-\propi guarantee with a 1.1\% accuracy decrease.
To obtain an $\varepsilon$-\propi guarantee using DP-SGD based approaches, we can compute their $\varepsilon$-DP guarantee (as described in Proposition 2 of \cite{ref_88}).
The DP-SGD baselines can not provide a $0$-DB guarantee. 
The closest guarantee they obtain is a $0.02$-DP guarantee, which requires adding a very high noise (the standard deviation of the noise distribution is $10,000$). 
The accuracy decrease of DP-SGD, for $\varepsilon\in\{0.02,0.2,1\}$, is on average 42.3\%, 14.2\%, and 14.3\%, respectively (29.8x, 10.7x, and 12.9x higher than \tool). 
The accuracy decrease of ALIBI for the same DP guarantees is on average 28.6\%, 19.3\%, and 2.7\% (20.1x, 14.5x, and 2.4x higher than \tool). 
Even worse, on some networks (Crypto 30$\times$4, Twitter 50$\times$2, and Twitter 100$\times$2), the baselines train networks whose accuracy is 50\% (like the accuracy of a random classifier).  
If we ignore these three networks, DP-SGD's accuracy decrease is 40.13\%, 3.99\%, and 3.9\%, which is still significantly higher than \tool (by 28.3x, 3.0x, and 3.5x), while ALIBI's decrease is 23.4\%, 10.8\%, and 2.7\% (higher by 16.4x, 8.2x, and 2.4x). 
As expected, \tool's analysis time is longer than the baselines. On average, it takes 4.8 hours to compute the \propa, required for the repair (training all networks takes half an hour on average when parallelizing over eight GPUs). 
In contrast, DP-SGD takes only 19.9 seconds and ALIBI takes 98.2 seconds. 
Although the baselines are significantly faster, their training algorithms are coupled to specific privacy budget, and if a user wishes to update the privacy guarantee, they need to retrain the classifier. In contrast, \tool's private access can be easily configured with any $\varepsilon$ guarantees (the \propa are computed once for a classifier).
Further, as we show, computing the \propa enables us to obtain a small decrease to the network's accuracy, unlike the baselines (for the same ~\propi guarantees).

}

\begin{table}[t]
\small
\begin{center}
\caption{The accuracy of DP-SGD and ALIBI ($Acc_{w/o-p}$ is the model's accuracy without privacy protection).}
\begin{tabular}{lll ccccc cccc}
\toprule
Dataset & Model & $Acc_{w/o-p}$   & \multicolumn{4}{c}{DP-SGD} & \multicolumn{4}{c}{ ALIBI}\\
 \cmidrule(lr){4-7} \cmidrule(lr){8-11}  
        &       &         &  $Acc$              &  $Acc$              & $Acc$           & $CT$   &  $Acc$           & $Acc$            & $Acc$         & $CT$ \\  
        &       &         &  $[\%]$            & $[\%]$              & $[\%]$          & $[s]$ &  $[\%]$          &  $[\%]$          & $[\%]$          & $[s]$ \\  
        &       &         & $\varepsilon=0.02$ &  $\varepsilon=0.2$  & $\varepsilon=1$   & &  $\varepsilon=0.02$ &  $\varepsilon=0.2$  & $\varepsilon=1$  & \\  

\midrule
Crypto   & 2$\times$50    &  99.7   &     49.8  &  96.6  &  96.6  &  8.2  &   48.5 &   48.5  &   92.8    & 19.8       \\ 
Crypto   & 2$\times$100   &  99.7   &    49.3  &  96.5  &  96.6  &  7.9  &   80.8 &   96.7  &   96.9    & 20.8       \\   
Crypto   & 4$\times$30    &  99.8   &      34.5  &  49.3  &  49.3  &  11.4  &   49.3 &   49.3  &   95.6    & 22.8       \\     
\midrule
Twitter   & 2$\times$50   &  89.7   &      51.9  &  51.4  &  51.5  &  25.1  &   51.4 &   51.4  &   87.6    & 169.2       \\  
Twitter   & 2$\times$100  &  89.6   &      47.5  &  53.2  &  52.4  &  28.8  &   51.4 &   52.6  &   87.6    & 159.6       \\ 
\midrule
Adult   & 2$\times$50     &  83.1   &        23.6  &  76.3  &  76.3  &  24.4  &   67.4 &   77.4  &   82.2    & 149.4       \\     

Adult   & 2$\times$100    &  83.3   &        76.3  &  76.4  &  76.4  &  20.8   &   61.5 &   75.3  &   81.9    & 103.6       \\ 

Adult   & $conv$          &  79.2   &        46.2  &  76.4  &  76.4  &  24.9   &   23.6 &   74.6  &   76.4    & 139.8       \\ 
\midrule
Credit   & 2$\times$50    &  81.9   &        26.2  &  78.4  &  78.5  &  22.4  &   78.4 &   78.4  &   78.6    & 96.4       \\ 
Credit   & 2$\times$100   &  81.8   &        71.8  &  78.5  &  78.5  &  21.1  &   63.7 &   73.1  &   79.9    & 81.8       \\ 
Credit   & $conv$         &  80.8   &        25.2  &  78.4  &  78.5  &  24.7  &   78.4 &   78.4  &   78.4    & 116.8       \\     
\bottomrule        
\end{tabular}
    \label{tab:results2}
\quad
\end{center}
\end{table}

\paragraph{Illustration of the \propa}
We next illustrate the reason that the \propa enables \tool to provide \propi-access to a network classifier with a minor accuracy decrease.
We consider the 2$\times$50 network for Twitter and let \tool compute the \propa of each class $\beta_0$ and $\beta_1$.
 \Cref{fig::eval1} plots the classification confidence of each input in the test set, where inputs labeled as $0$ are shown on the left  plot and inputs labeled as $1$ are shown on the right plot. The plots show in a green dashed line the \propa, in blue points the test points whose confidence is greater than the \propa and in red points the test points whose confidence is smaller or equal to the \propa. The plots 
 show that the confidence of the vast majority of points in the test set is over the \propa. Recall that for these points, \tool does not employ the exponential mechanism and thereby it obtains a minor accuracy decrease.

\begin{table}[t]
\small
\begin{center}
\caption{Ablation study over \tool's components. 
 Hyp stands for Hyper-network, MD for Matching Dependencies, RiS for Relax-if-Similar, and BaB for Branch-and-Bound.}
\begin{tabular}{lllcccccccccc}
    \toprule
Dataset & Model & Variant & Hyp & MD & RiS & BaB & $\beta_{0}$ & $T_0$ & $\beta_{1}$ & $T_1$ & $Acc_{\varepsilon=0}$ & $Acc_{\varepsilon=1}$ \\
        &       &         &       &    &     &     &             & $[h]$ &             & $[h]$ & [\%]    & [\%] \\
\midrule
 Twitter& 2$\times$50   & Naive     &  \xmark & \xmark  & \xmark  & \xmark                                 & -       & 8.00        & -         & 8.00        & -         & -                      \\    
        & $Acc_{w/o-p}$  & H     &  \cmark & \xmark  & \xmark  & \xmark                                 & 1.62      & 8.00        & 1.49        & 8.00        & 85.4        & 85.9                      \\
         & 89.7\%       & HM    &  \cmark & \cmark  & \xmark  & \xmark                                 & 1.22      & 0.18     & 1.20        & 0.38     & 87.5       & 87.8                      \\
        &        & HMRiS   &  \cmark & \cmark  & \cmark  & \xmark                                 &1.25       & 0.17    & 1.26         & 0.11     & 87.3       & 87.7                      \\
         &        & HMB &  \cmark & \cmark  & \xmark  & \cmark                                 & 0.46      & 4.56    & 0.43         & 5.26     & 89.3        & 89.4                      \\

        &        & \tool &  \cmark & \cmark  & \cmark  & \cmark                                 & 0.51       & 1.16    & 0.46         & 2.81     & 89.3        & 89.4                      \\
\midrule
 Adult  & conv & Naive     &  \xmark & \xmark  & \xmark  & \xmark                                 & -      & 8.00        & -         & 8.00         & -         & -                      \\
        & $Acc_{w/o-p}$      & H     &  \cmark & \xmark  & \xmark  & \xmark                                 & 0.96      & 8.00        & 0.58       & 7.50        & 74.8     & 75.3                      \\
          & 79.2\%       & HM    &  \cmark & \cmark  & \xmark  & \xmark                              & 0.86     & 0.04      & 0.58       & 0.04       & 75.4     & 75.9                      \\
        &      & HMRiS    &  \cmark & \cmark  & \cmark  & \xmark                                 & 0.87      & 0.03      & 0.58       & 0.04       & 75.3      & 75.9                      \\
         &      & HMB   &  \cmark & \cmark  & \xmark  & \cmark                                 & 0.33     & 1.65      & 0.24        &4.05       & 80.3      & 80.1                      \\
        &      & \tool   &  \cmark & \cmark  & \cmark  & \cmark                                & 0.40      & 0.76      & 0.36        & 0.81        & 80.1     & 80.1                      \\
\bottomrule
\end{tabular}
    \label{tab:ablation}
\quad
\end{center}
\end{table}

 \begin{figure*}[t]
    \centering
  \includegraphics[width=1\linewidth, trim=0 385 0 0, clip,page=5]{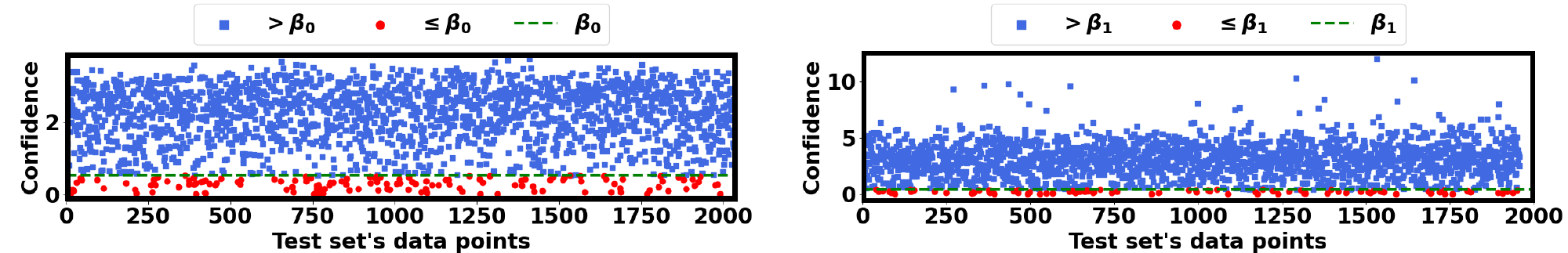}
    \caption{The \propa and the classification confidence of the test set's data points, for Twitter 2$\times$50. }
    \label{fig::eval1}
\end{figure*}

\paragraph{Ablation tests}  
Next, we evaluate the effectiveness of \tool's components. We focus on the Twitter 2$\times$50 network and on the Adult convolutional network. 
We consider five variants of \tool:
\begin{itemize}[nosep,nolistsep] 
\item Naive: The naive approach described in~\Cref{sec:overview_opt}, which computes the \propa of every $x_D \in D$ by encoding the problem $P_{x_D}$ as a MILP (like \tool, it has 32 parallel workers). 
\item Hyper-network (H): This variant constructs a hyper-network and solves the problem of~\Cref{dphyp} for $S=D$ by encoding it as a MILP (as described in~\Cref{sec:overview_hyper}). 
\item +Matching dependencies (HM): This variant is H with the matching dependencies. 
\item +Relax-if-similar (HMRiS): This variant is HM with the relax-if-similar technique. 
\item +Branch-and-bound (HMB): This variant is HM with our branch-and-bound.

\end{itemize}
We run every approach with an eight hour timeout to compute the \propa of each class
 $\beta_0$ and $\beta_1$. 
We measure the computation time $T_0$ and $T_1$ and the accuracy decrease for $\varepsilon=0$ and $\varepsilon=1$. 
\Cref{tab:ablation} summarizes the results. 
The results show that Naive does not compute an \propa within the timeout (marked by -). 
In fact, it iterates over fewer than 5\% of the networks. 
Unlike \tool's anytime algorithm, Naive must solve all problems in $\{P_{x_D} \mid x_D \in D \}$ to provide a sound \propa.
The hyper-network variant terminates but provides very loose \propa, approximately 2.57x larger than those obtained by \tool and it further requires 7.2x more time than \tool. 
This looseness is caused by the overapproximation of the hyper-network. 
The computation time is very long because there are no matching dependencies, making the optimization very slow, often reaching the timeout. 
As expected, the hyper-network and matching dependencies variant converges faster to the \propa: it is faster by 99.9x than the hyper-network variant.   
However, its \propa remain high, by 2.16x than those computed by \tool (because it analyzes a single hyper-network, which introduces an overapproximation error). 
The variant that also employs relax-if-similar reduces the analysis time by 1.64x compared to the previous variant, with a minor increase to the bound (1.02x). 
The variant that adds branch-and-bound computes \propa that are 2.66x tighter than the previous variant and 1.2x tighter than \tool. 
However, its computation time is 3.2x longer than \tool. 
This is expected, since the only difference between this variant and \tool is the relax-if-similar technique, which accelerates the computation at the cost of overapproximation. 
When integrating each variant (except Naive, which did not terminate) with the two privacy budgets $\varepsilon=0$ and $\varepsilon=1$, the accuracy decrease is 
4.4\%, 3.0\%, 3.2\%, -0.4\% and -0.3\%, for $\varepsilon=0$, and  
 3.9\%, 2.6\%, 2.7\%, -0.3\% and -0.3\%, for $\varepsilon=1$. 
A negative decrease means that \tool improves the accuracy.

\section{Related Work}
\label{sec:related_work}

\paragraph{Privacy of neural networks}
Several works propose approaches to reduce privacy leakage of neural networks. 
Some propose specialized architectures to maintain privacy~\cite{ref_16,ref_17}. 
Others integrate regularization during training to limit leakage~\cite{ref_18, ref_19}. 
Others rely on federated distributed learning~\cite{ref_21,ref_34,ref_57}. 
Another approach is to train differentially private (DP) networks~\cite{ref_22,ref_36,ref_37,ref_58,ref_59,ref_60}. 
Many works considered variants of DP to ensure privacy. For example, in individualized privacy assignment, different privacy budgets can be allocated to different users~\cite{ref_94,ref_95}. 
In local differential privacy, a data owner adds randomization to their data before it leaves their devices~\cite{ref_60,ref_97}. Renyi differential privacy relaxes differential privacy based on the Renyi divergence~\cite{ref_98,ref_99}. Local differential classification privacy proposes a new privacy property inspired by local robustness~\cite{ref_8}. Homomorphic-encryption approaches encrypt data before obtaining DP guarantees~\cite{ref_100,ref_101}. 
\paragraph{Multiple network analysis} 
\tool analyzes a large set of similar networks. 
There is prior work on analysis of several networks. 
Some works compute the output differences of two networks~\cite{ref_1,ref_2}
or the output differences of a set of inputs~\cite{ref_102}. 
Others study incremental verification, analyzing the differences of a network and its slightly modified version~\cite{ref_3}. 
Others propose proof transfer for similar networks for verifying local robustness~\cite{ref_4}. 
Some global robustness verifiers compare the outputs of a network given an input and given its perturbed version~\cite{ref_6,ref_7,ref_5}. 

\paragraph{Clustering in neural network verification} 
\tool expedites its analysis by clustering close networks and computing their hyper-networks. 
Some verifiers accelerate local robustness analysis by grouping neurons into subgroups~\cite{ref_9, ref_10}, while others compute centroids to partition inputs and establish global robustness bounds~\cite{ref_11}. 
\section{Conclusion}
\label{sec:conclusions_and_discussion}
We present \tool, a system that creates \propi label-only access for a neural network classifier with a minor accuracy decrease.   
\tool computes the \propa, which overapproximates the set of inputs that violate \propi 
and adds noise only to the labels of these inputs. 
To compute the \propa, \tool 
relies on several techniques: constraint solving, hyper-networks abstracting a large set of networks, and a novel branch-and-bound technique. 
To further scale, it prunes the search space by
computing matching dependencies 
and employing linear relaxation. 
Our experimental evaluation shows that our verification analysis enables \tool to provide a $0$-\propi guarantee with an accuracy decrease of 1.4\%. For more relaxed \propi guarantees, \tool can reduce the accuracy decrease to 1.2\%, whereas existing DP approaches lead to an accuracy decrease of 12.7\% to provide the same \propi guarantees. 
\section*{Acknowledgements} We thank the anonymous reviewers for their feedback. This research was supported by the Israel Science Foundation (grant No. 2605/20). 

\section*{Data-Availability Statement}
Our code is available at \url{https://github.com/ananmkabaha/LUCID.git}. 

\bibliography{bib}
\newpage
\appendix

\section{Proofs}\label{sec:proofs}
\ftb*
\begin{proof}
Generally, given a dataset $D$, a set of outputs $\mathcal{R}$ and a utility function $u:\mathcal{D} \times \mathcal{R} \rightarrow \mathbb{R}$,
the exponential mechanism $\mathcal{A}$ samples and returns $r\in \mathcal{R}$ with probability $\frac{\text{exp}(\varepsilon u(D,r)/(2\Delta u))}{\sum_{r'\in \mathcal{R}} \text{exp}(\varepsilon u(D,r')/(2\Delta u))}$.
In our setting, the set of outputs is the label set $\mathcal{R}=C$.
To prove \propi, we show that for any dataset $D'$ adjacent to $D$ (i.e., differing by one data point) and for every possible observed output $\mathcal{O} \subseteq \text{Range}(\mathcal{A})$, the following holds:
$e^{-\varepsilon} \cdot \Pr[\mathcal{A}(D') \in \mathcal{O}] \leq\Pr[\mathcal{A}(D) \in \mathcal{O}] \leq e^{\varepsilon} \cdot \Pr[\mathcal{A}(D') \in \mathcal{O}]$.
In our setting, the observed outputs $\mathcal{O} \subseteq \text{Range}(\mathcal{A})$ are sets of a single class $\{c\}\subseteq C$.
 Thus, we prove that for any dataset $D'$ adjacent to $D$ and for every $c \in C$, the following holds:  
 $$e^{-\varepsilon} \cdot \Pr[\mathcal{A}(D',u_{x,\mathcal{T},\widetilde{N}},C) =c] \leq\Pr[\mathcal{A}(D,u_{x,\mathcal{T},\widetilde{N}},C) =c] \leq e^{\varepsilon} \cdot \Pr[\mathcal{A}(D',u_{x,\mathcal{T},\widetilde{N}},C) =c]$$ 
We prove $\Pr[\mathcal{A}(D,u_{x,\mathcal{T},\widetilde{N}},C) =c] \leq e^{\varepsilon} \cdot \Pr[\mathcal{A}(D',u_{x,\mathcal{T},\widetilde{N}},C) =c]$ (the proof for the other inequality is similar).
The proof is very similar to the proof of~\cite[Theorem 3.10]{ref_87}.
To simplify notation, we write $u\triangleq u_{x,\mathcal{T},\widetilde{N}}$. 
We remind that $\mathcal{A}$ is invoked with $\Delta u=1$.


\begin{multline*}
\frac{\Pr[\mathcal{A}(D,u, C) = c]}{\Pr[\mathcal{A}(D',u,C) = c]}=\frac{\left( \frac{\text{exp}(\varepsilon u(D,c)/2)}{\sum_{c'\in C} \text{exp}(\varepsilon u(D,c')/2)} \right) }{\left( \frac{\text{exp}(\varepsilon u(D',c)/2)}{\sum_{c'\in C} \text{exp}(\varepsilon u(D',c')/2)} \right) }=\\ 
\text{exp}\left(\frac{\varepsilon(u(D,c)-u(D',c))}{2}\right)\cdot\left(\frac{\sum_{c'\in C} \text{exp}(\varepsilon u(D',c')/2)}{\sum_{c'\in C} \text{exp}(\varepsilon u(D,c')/2)}\right)
\leq 
\text{exp}(\varepsilon)
\end{multline*}
The last transition follows from:
\begin{itemize}[nosep,nolistsep]
  \item The left term is equal to $\text{exp}\left(\frac{\varepsilon}{2}\right)$: This follows since by the definition of our utility function, for every $D'$ and $c$, we have $u(D,c)-u(D',c) \leq 1$.
\item The right term is equal to  $\text{exp}\left(\frac{\varepsilon}{2}\right)$: This follows by 
the quotient inequality which states that if we have $K$ values $\alpha_1,\ldots,\alpha_K$ and $K$ respective values $\beta_1,\ldots,\beta_K$ and for every $i\in[K]$, we have $\alpha_i / \beta_i \leq \gamma$, for some $\gamma$, then
$\frac{\alpha_1+\ldots + \alpha_K} {\beta_1+\ldots+\beta_K} \leq \gamma$. In our right term, we have for every $c'\in C$,  $\alpha_{c'}=\exp(\epsilon u(D',c')/2)$ and $\beta_{c'}= \exp(\epsilon u(D,c')/2)$. 
That is  $\frac{\alpha_{c'}}{\beta_{c'}} = \frac{\exp(\epsilon u(D',c')/2)}{\exp(\epsilon u(D,c')/2)}=
 \exp(\epsilon (u(D',c')- u(D,c'))/2)$. Like in the left term, since  $u(D',c')-u(D,c') \leq 1$
 we get $\alpha_{c'}/\beta_{c'} \leq \exp(\epsilon/2)\triangleq \gamma$, and the claim follows. 
\end{itemize}
 \end{proof}

\ftc*
\begin{proof}[Proof Sketch]
By the correctness of \boundtool, its computed \propa provides a sound overapproximation to the set of leaking inputs. Thus, if an input $x$ has classification confidence above the respective \propa, its query satisfies $0$-\propi and thus also $\varepsilon$-\propi. For such inputs, \tool does not add noise, and thereby does not lose accuracy.  
Otherwise, the exponential mechanism is employed and thereby $\varepsilon$-\propi is guaranteed by~\Cref{lemma::exponential_mechanism}. 
\end{proof}

\newpage
\section{Evaluation: Dataset Description}\label{sec:appeval}
In this section, we describe our evaluated datasets. 

\paragraph{Datasets}
We evaluate \tool over four datasets:
 \begin{itemize}[nosep,nolistsep]
    \item Cryptojacking~\cite{ref_43} (Crypto): 
    Cryptojacking refers to websites that exploit user resources for cryptocurrency mining. This dataset includes 2,000 malicious websites and 2,000 benign websites, where 2,800 are used for training and 1,200 for testing. Each website has seven features (e.g., websocket) and is labeled as malicious or benign.
    \item Twitter Spam Accounts~\cite{ref_44} (Twitter): 
 Twitter spam accounts refers to accounts that distribute unsolicited messages to users. This dataset includes 20,000 spam accounts and 20,000 benign accounts, split into 36,000 accounts for training and 4,000 for testing. Each data point has 15 features (e.g., followers and tweets) and is labeled as spam or benign.
    \item Adult Census~\cite{ref_45} (Adult): 
 This dataset is used for predicting whether an adult's annual income is over \$50,000. It includes 48,842 data points, split into 32,561 for training and 16,281 for testing. Each data point has 14 features (e.g., age, education, occupation, and working hours) and is labeled as yes or no. 
    \item Default of Credit Card Clients~\cite{ref_47} (Credit): 
The Taiwan Credit Card Default dataset is used for predicting whether the default payment will be paid. It 
 includes 30,000 client records, split into 24,000 for training and 6,000 for testing.
 A data point has 23 features (e.g., bill amounts, age, education, marital status) and is labeled as yes or no.
  \end{itemize}

\end{document}